\documentclass{article}

\usepackage{microtype}
\usepackage{graphicx}
\usepackage{subfigure}
\usepackage{booktabs} 

\usepackage{hyperref}

\usepackage[accepted]{icml2023}

\usepackage{amsmath}
\usepackage{amssymb}
\usepackage{mathtools}
\usepackage{amsthm}
\usepackage{enumitem}

\usepackage{soul}
\setstcolor{red}

\usepackage[capitalize,noabbrev]{cleveref}

\usepackage[figuresright]{rotating}
\usepackage{multirow}
\usepackage{epstopdf}
\usepackage{threeparttable}
\usepackage{array}  
\usepackage{makecell}
\usepackage{tabularx}
\usepackage{eqparbox} 
\renewcommand{\algorithmiccomment}[1]{\hfill\eqparbox{COMMENT}{\# \texttt{#1}}}

\theoremstyle{plain}
\newtheorem{theorem}{Theorem}[section]
\newtheorem{proposition}[theorem]{Proposition}
\newtheorem{lemma}[theorem]{Lemma}
\newtheorem{corollary}[theorem]{Corollary}
\theoremstyle{definition}
\newtheorem{definition}[theorem]{Definition}
\newtheorem{assumption}[theorem]{Assumption}
\theoremstyle{remark}
\newtheorem{remark}[theorem]{Remark}

\usepackage[textsize=tiny]{todonotes}

\icmltitlerunning{TnALE: Solving Tensor Network Structure Search with Fewer Evaluations}

\begin{document}

\twocolumn[
\icmltitle{Alternating Local Enumeration (TnALE):\\ Solving Tensor Network Structure Search with Fewer Evaluations}



\icmlsetsymbol{equal}{*}

\begin{icmlauthorlist}
\icmlauthor{Chao Li}{aip}
\icmlauthor{Junhua Zeng}{equal,gut,aip}
\icmlauthor{Chunmei Li}{equal,heu,WASEDA}
\icmlauthor{Cesar Caiafa}{iar,aip}
\icmlauthor{Qibin Zhao}{aip}
\end{icmlauthorlist}

\icmlaffiliation{aip}{RIKEN-AIP, Tokyo, Japan}
\icmlaffiliation{heu}{College of Information and Communication Engineering, Harbin Engineering University, Harbin, China}

\icmlaffiliation{gut}{School of Automation, Guangdong University of Technology, Guangzhou, China}

\icmlaffiliation{WASEDA}{Department of Computer Science and Communications Engineering, WASEDA University, Tokyo, Japan}

\icmlaffiliation{iar}{Instituto Argentino de Radioastronomía, CONICET CCT La Plata/CIC-PBA/UNLP, V. Elisa, ARGENTINA
}

\icmlcorrespondingauthor{Qibin Zhao}{qibin.zhao@riken.jp}
\icmlcorrespondingauthor{Chao Li}{chao.li@riken.jp}

\icmlkeywords{Tensor Network, Tensor Decomposition, Model Selection, High-Dimensional Data}

\vskip 0.3in
]



\printAffiliationsAndNotice{\icmlEqualContribution} 

\begin{abstract}
    Tensor network (TN) is a powerful framework in machine learning, but selecting a good TN model, known as TN structure search (TN-SS), is a challenging and computationally intensive task.
    The recent approach TNLS~\cite{li2022permutation} showed promising results for this task. 
    However, its computational efficiency is still unaffordable, requiring too many evaluations of the objective function. 
    We propose TnALE, a surprisingly simple algorithm that updates each structure-related variable alternately by local enumeration, \emph{greatly} reducing the number of evaluations compared to TNLS.
    We theoretically investigate the descent steps for TNLS and TnALE, proving that both the algorithms can achieve linear convergence up to a constant if a sufficient reduction of the objective is \emph{reached} in each neighborhood.
    We further compare the evaluation efficiency of TNLS and TnALE, revealing that $\Omega(2^K)$ evaluations are typically required in TNLS for \emph{reaching} the objective reduction, while ideally $O(KR)$ evaluations are sufficient in TnALE, where $K$ denotes the dimension of search space and $R$ reflects the \emph{``low-rankness''} of the neighborhood.
    Experimental results verify that TnALE can find practically good TN structures with vastly fewer evaluations than the state-of-the-art algorithms.
    Our code is available at \url{https://github.com/ChaoLiAtRIKEN/TnALE}.
\end{abstract}
\section{Introduction}
Tensor network (TN) has been widely applied to solving high-dimensional problems in both machine learning~\cite{anandkumar2014tensor,novikov2015tensorizing,zhe2015scalable,glasser2019expressive,kossaifi2020tensor,miller2021tensor,pmlr-v139-richter21a,malik2022more} and quantum physics~\cite{orus2019tensor,felser2021quantum}.
The success of \emph{AlphaTensor}~\cite{AlphaTensor2022} once again confirmed the usefulness of tensors in various fields.
TN practitioners, on the other hand, have to 
face in practice challenging
problems associated with model selection, known as \emph{TN structure search (TN-SS)}, for example: (1) how to determine the TN-ranks?; (2) should we prefer tensor-train~(TT,~\citealt{oseledets2011tensor}), tensor-ring~(TR,~\citealt{zhao2016tensor}) or other TN topology?; (3) how to relate the tensor modes to core tensors of a TN~(the permutation problem,~\citealt{li2022permutation}), and so on.
Unfortunately, some of these problems have been proven to be NP-hard~\cite{hillar2013most}\footnote{For instance, it proves that determining the optimal ranks for Tucker decomposition~\cite{tucker1966some} is NP-hard.}, and most of them suffer from the \emph{``combinatorial explosion''}\footnote{It means the rapid growth of TN structure searching space
due to the combinatorics of ranks, 
topologies, permutations, etc. } so that the brute force search is not a viable option in practice.

Several works have put effort into different aspects of 
TN-SS (see Section \ref{sec:relworks}), 
but many of the methods are restricted to 
specific tasks or work poorly in practice, so a general and efficient TN-SS method is needed.
Recently,~\citet{li2022permutation} proposed an algorithm dubbed TNLS, which addressed the rank and permutation selection problem for TNs.
However, its computational complexity is high since it requires evaluating the objective function on a large number of structure candidates.

To address this issue, we accelerate TNLS by equipping the algorithm with \textbf{A}\textit{lternating} \textbf{L}\textit{ocal} \textbf{E}\textit{numeration} --- a surprisingly simple but efficient searching method in neighborhoods.
The new algorithm, named \emph{TnALE}, can improve the evaluation efficiency of TNLS \emph{greatly}.
To be specific, TnALE follows the ``local-searching'' scheme as TNLS but alternately updates each structure-related variable by enumerating its values within a neighborhood.
The intuition is that the alternately updating avoids the combinatorial explosion and the enumeration in neighborhoods guarantees the non-increasing of values of the objective in the search.
On the other hand, the original TNLS applies random sampling, causing the majority of samples would not provide helpful information for decreasing the value of the objective function.

The theoretical advantage of TnALE is also clear.
We prove that, with new-defined \emph{discrete} convexity-related assumptions of the objective function, both TNLS and TnALE can achieve a linear convergence up to a constant if a sufficient reduction 
of the objective function is reached in each neighborhood~(Theorem~\ref{thm:rateFixedP}).
We also prove that the number of evaluations required in TNLS would grow \emph{exponentially} with the dimension of search space (Prop.~\ref{thm:TNLSSammpling}), with respect to the dimension of TN-ranks and the TN order, while TnALE shows a \emph{linear} growth in the ideal case (Prop.~\ref{thm:TnALESampling}).
Our analysis reveals that such
an improvement in the evaluation efficiency essentially comes from the \emph{low-rankness} of the optimization landscape in neighborhoods, attributed to the close relationship between TnALE and cross-approximation methods for matrices and tensors \cite{tyrtyshnikov2000incomplete, oseledets2010tt, sozykin2022ttopt}.
Numerically, extensive experiments on both synthetic and real-world data are implemented to assess the evaluation efficiency and the superior performance provided by TnALE.

Our main contributions 
can be summarized as follows:
\begin{itemize}
[noitemsep,topsep=0pt,parsep=5pt,partopsep=0pt]
    \item We propose TnALE, a novel algorithm that \emph{greatly} reduces the computational cost for the task of TN structure search (TN-SS);
    \item We establish for the first time the convergence analysis for both TNLS~\cite{li2022permutation} and TnALE, and rigorously prove their evaluation efficiency.
\end{itemize}

\subsection{Related Works}{\label{sec:relworks}
}

\noindent\textbf{Tensor network structure search (TN-SS).}
TN-SS can be specified into three 
sub-problems:
(1) TN-rank selection (TN-RS)~\citep{rai2014scalable,zhao2015bayesian,yokota2016smooth,cheng2020novel,mickelin2020algorithms,cai2021blind,hawkins2021bayesian,long2021bayesian,sedighin2021adaptive,yin2022batude,ghadiri2023Approximately}; (2) TN-topology selection (TN-TS)~\cite{hashemizadeh2020adaptive,haberstich2021active,nie2021adaptive,falco2023geometry,hikihara2023automatic,kodryan2023mars,liu2023adaptively}; and (3) TN-permutation selection (TN-PS)~\cite{acharya2022qubit,chen2022one}.
Recently, some methods 
to solve the TN-SS problem were
developed via discrete optimization~\cite{hayashi2019exploring,hashemizadeh2020adaptive,li2020evolutionary,li2021heuristic,li2022permutation,solgi2022evolutionary}. 
Although these methods typically achieve higher precision than their counterparts in practice, they suffer from the expensive computational cost and the lack of theoretical analysis.
Our work follows the TNLS algorithm~\cite{li2022permutation} in this direction but develops a new approach to improve its computational efficiency and fill in the missing theoretical analysis for convergence and evaluation efficiency.

\noindent\textbf{Finding the 
extreme entry value within a tensor.}
As discussed later in Section~\ref{sec:samplingEfficiency}, the \emph{alternating local enumeration} 
method is technically equivalent to finding the minimum entry within a multidimensional landscape.
As such, our work is strongly related to the recently published method 
TTOpt~\cite{sozykin2022ttopt}, which finds the near-maximum entry of a tensor by cross-sampling~\cite{tyrtyshnikov2000incomplete,zhang2019cross} in TT topology~\cite{oseledets2010tt}.
Compared to TTOpt, the proposed TnALE 
recursively finds
the extreme entry within a tensor associated with the neighborhood rather than the global search space, so that TnALE can handle the situation of \emph{infinite} candidates (entries) in the optimization.




\section{Preliminaries}
In this section, we first summarize notations and review several central concepts related to the tensor network structure search (TN-SS).
Then, we provide a quick review of TNLS~\cite{li2022permutation}, a recently proposed algorithm for solving TN-SS, and point out that
\emph{TNLS suffers from the curse of dimensionality} in evaluation
efficiency.


\subsection{Notations}
Throughout the paper, we typically use blackboard letters to denote sets of objects, \textit{e.g.}, $\mathbb{G,F}$.
In particular, $\mathbb{R}$, $\mathbb{Z}_{+}$ and $\mathbb{Z}_{\geq{}0}$ denote real numbers, positive integers and non-negative integers, respectively.
We use boldface letters to denote vectors and matrices,~\textit{e.g.}, $\mathbf{x}\in\mathbb{Z}_+^{K}$ and $\mathbf{A}\in\mathbb{R}^{I\times{}J}$.
For tensors of \emph{arbitrary} order, we denote them using calligraphic letters,~\textit{e.g.}, $\mathcal{A,B}\in\mathbb{R}^{I_1\times{}I_2\times\cdots\times{}I_N}$.
Given a vector, such as $\mathbf{x}\in\mathbb{Z}_+^K$, $\Vert\mathbf{x}\Vert$ and $\Vert\mathbf{x}\Vert_\infty$ denote the $l_2$-norm and $l_\infty$-norm of $\mathbf{x}$, respectively.
The norms are also directly applied to matrices and tensors by thinking of them as vectors.
Following the notational conventions, $\vert{}x\vert$ denotes the absolute value of $x$ if $x\in\mathbb{R}$ is a scalar, while $\vert\mathbb{A}\vert$ denotes the cardinality if $\mathbb{A}$ is a set.
For the normed vector spaces armed with $\Vert\mathbf{x}\Vert_\infty$, we use $B_\infty(\mathbf{x},r_\mathbf{x})$ to denote the neighborhoods centered at $\mathbf{x}$ with the radius $r_\mathbf{x}>0$.
For any two functions $f:\mathbb{B}\rightarrow\mathbb{C}$ and $g:\mathbb{A}\rightarrow{}\mathbb{B}$, the operation $f\circ{}g:\mathbb{A}\rightarrow{}\mathbb{C}$ denotes the function composition.

\subsection{Tensor Network Structure Search~(TN-SS)}
We consider the \emph{tensor network}~(TN,~\citealp{Ye2019Tensor}) as a set of real tensors of the dimension $I_1\times{}I_2\times\cdots\times{}I_N$,
denoted $TNS(G,\mathbf{r})$, whose elements are in the form of contractions of \emph{core tensors}~\cite{cichocki2017tensor}, associated to the TN structure modeled by the pair $(G,\mathbf{r})$, where $G=(V,E)$ denotes a simple \emph{graph} of $N$ vertices modeling the \emph{TN-topology}~\cite{li2020evolutionary} 
and 
$\mathbf{r}\in\mathbb{Z}_{+}^{\vert{}E\vert}$ is a vector, whose entries are edge labels of $G$ corresponding to the \emph{TN-ranks}~\cite{Ye2019Tensor}.

\emph{Tensor network structure search} (TN-SS) aims generally to find the most compressed TN models for computational purposes while preserving the models' expressivity.
Furthermore, the most compressed TN models also imply the potential advantage for the generalization capability in learning tasks~\cite{khavari2021lower}.
Suppose a dataset $D$ and the task-specific loss function $\pi_D:\mathbb{R}^{I_1\times{}I_2\times\cdots\times{}I_N}\rightarrow{}\mathbb{R}_+$ involving $D$.
TN-SS is to solve the following bi-level discrete optimization problem
\begin{equation}
    \begin{split}
        \min_{(G,\mathbf{r})\in{}\mathbb{G}\times{}\mathbb{F}_G}{}\left(\phi(G,\mathbf{r})+\lambda\cdot\min_{\mathcal{Z}\in{}TNS(G,\mathbf{r})}\pi_D(\mathcal{Z})\right)\label{eq:basicModel}
    \end{split},
\end{equation}
where $G\in\mathbb{G}$ is a graph of $N$ vertices and $K$ edges, \mbox{$\mathbf{r}\in\mathbb{F}_G\subseteq\mathbb{Z}_+^K$}, $\phi:\mathbb{G}\times{}\mathbb{Z}_+^K\rightarrow{}\mathbb{R}_+$ represents the function measuring the model complexity of a TN whose structure is modeled by $(G,\mathbf{r})$, and $\lambda>0$ is a tuning parameter.
The intuition of~\eqref{eq:basicModel} is that, the inner minimization is to evaluate the task-specific expressivity for a TN structure, while the outer minimization is to find the optimal structure for the task by balancing the complexity and the expressivity 
of a TN model.

We remark that the formulation~\eqref{eq:basicModel} can be specified as different TN-SS sub-problems by restricting the feasible set $\mathbb{G}\times{}\mathbb{F}_G$ into different forms:
for TN-PS, we specify $\mathbb{F}_G=\mathbb{Z}_+^{K}$ and $\mathbb{G}$ to be the set containing the isomorphic graphs to a ``template'' graph~\cite{li2022permutation};
for TN-RS, it typically restricts $G$ to be fixed, and only finds TN-ranks \textit{i.e.}, $\mathbb{F}_G=\mathbb{Z}_+^K$;
last for TN-TS, it relaxes $\mathbb{G}$ to be the set containing all possible simple graphs of $N$ vertices and $\mathbf{r}$ is set to be fixed~\cite{li2020evolutionary} or searchable~\cite{hashemizadeh2020adaptive}.
Note from~\citet{Ye2019Tensor,lirank} that TN-TS (with rank selection) essentially coincides with TN-RS of a ``fully-connected'' TN~\cite{zheng2021fully}.

\subsection{TNLS: a Discrete Optimization Approach to TN-PS}
Recently, \citet{li2022permutation} proposed an 
algorithm called TNLS for solving~\eqref{eq:basicModel} effectively by stochastic search.
The core process of TNLS is reviewed in Alg.~\ref{alg:TNLS}, from which we see that the candidate of the optimal TN structure is updated if the algorithm finds better structures within a neighborhood using \emph{random sampling}.
Although \citet{li2022permutation} illustrates the superiority of TNLS compared to its counterparts, the algorithm is still time-consuming as acknowledged in their work.
To understand the reason, we theoretically observe that TNLS suffers from the curse of dimensionality, due to the random sampling.
More specifically, we state that

\emph{under reasonable conditions, $\Omega(2^K/\epsilon)$ samples are required by random sampling (wrt. step 1 in Alg.~\ref{alg:TNLS}) for achieving a constant probability $Pr\geq{}\epsilon$ for decreasing the objective of~\eqref{eq:basicModel} in a neighborhood, where $\epsilon>0$ and $K$ denotes the dimension of the search space.}
\begin{algorithm}[tb]
  \caption{\emph{The core process} of TNLS~\cite{li2022permutation}.}
  \label{alg:TNLS}
\begin{algorithmic}
  \STATE {\bfseries Initialize:}
  Randomly choose a TN structure as the center of the neighborhood.
  \WHILE{not convergence}
  \STATE 1. Sampling $(G,\mathbf{r})$'s \underline{\emph{randomly}} in the neighborhood;
  \STATE 2. Evaluating the samples with the objective of~\eqref{eq:basicModel};
  \STATE 3. Updating the center if better samples are obtained;
  \ENDWHILE
  \STATE {\bfseries Output:} The center of the neighborhood.
\end{algorithmic}
\end{algorithm}

A formal statement is deferred to Prop.~\ref{thm:TNLSSammpling}.
It is known from Alg.~\ref{alg:TNLS} that each sampled structure should be evaluated by solving the inner minimization of~\eqref{eq:basicModel}, so the huge number of samples implies the prohibitive cost in computation.
To address this problem, we introduce a more sampling-efficient approach by reforming the random sampling in Alg.~\ref{alg:TNLS}, to accelerate the  TN-SS procedure with fewer evaluations.

\section{TnALE: Accelerating TNLS via Alternating Local Enumeration}\label{sec:methods}
We present now the new searching algorithm dubbed~\emph{TnALE} for solving the TN-SS problem.
Figure~\ref{fig:TnALE} demonstrates the key idea for TnALE.
In the rest of this section, we mainly focus on the technical details of \emph{alternating local enumeration (ALE)}, which replaces the random sampling operation in Alg.~\ref{alg:TNLS} as the key factor for algorithm acceleration.
A complete introduction of TnALE, including the pseudocode and other details, is given in Appendix~\ref{apd:sec:TnALE}.

In TnALE, we find better structures within a neighborhood by updating each structure-related variable alternately.
For instance, Figure~\ref{fig:ALS} illustrates how ALE solves the TN-PS problem, searching for the optimal ranks and permutations for tensor ring (TR) decomposition of order-$4$.
As shown in the initialization of panel (b), all structure-related variables, including $r_{\{1,2,3,4\}}$ and $G$, are initialized with the center of a given neighborhood.
To start the search, TN structures are sampled by enumerating all values of $r_1$ \emph{within} the neighborhood while fixing other variables.
Next, the sampled TN structures with varying $r_1$ are evaluated individually by calculating the objective of~\eqref{eq:basicModel}, and $r_1$ is subsequently updated right off by choosing the one with the minimum objective in the samples.
Following $r_1$, the same operation is applied to variables from $r_2$ to $G$ sequentially (see panel (b)).
After updating $G$, we turn the updating direction backward from $r_4$ to $r_1$,~\textit{i.e.}, in a ``round-trip'' manner (see panel (c)).
Overall, the ALE will be stopped if all variables are not changed anymore.
We empirically find that a \emph{one-time} ``round-trip'' is sufficient to reach a good TN structure for the next iteration.
Note that the operation of ALE for TN-RS and TN-TS is in the same fashion, except that the graph $G$ will be fixed in TN-RS or enumerated in differently-defined neighborhoods in TN-TS for considering all TN-topologies.
In TnALE, we follow~\citet{li2022permutation} to construct the neighborhood of TN structures. 

\begin{figure}
    \centering
    \includegraphics[width=1\columnwidth]{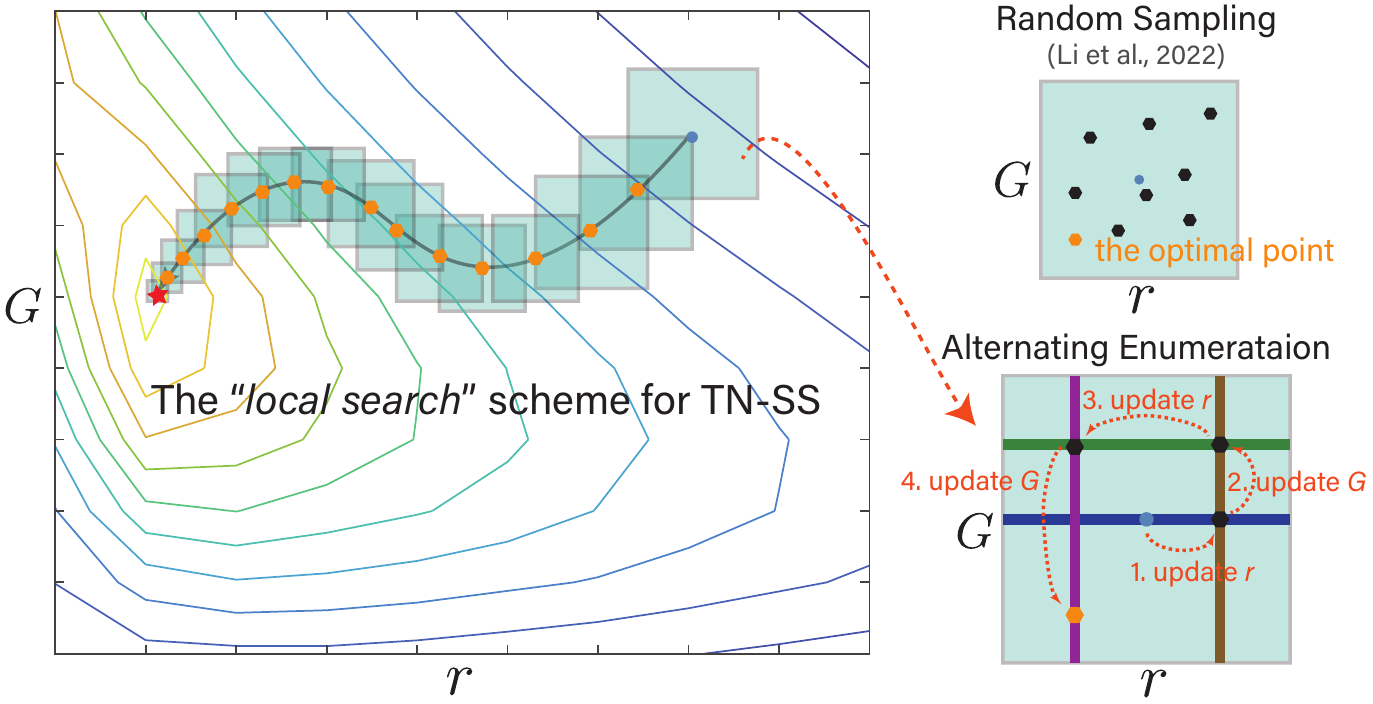}
    \vspace{-0.4cm}
    \caption{
    Schematic demonstration of \emph{TnALE} and its discrepancy from TNLS~\cite{li2022permutation}, where $r,\,G$ denote two structure-related variables for example, and the squares represent the neighborhoods. 
    The alternating (local) enumeration is further illustrated in detail in Figure~\ref{fig:ALS}.
    }
    \label{fig:TnALE}
    \vspace{-0.4cm}
\end{figure}

\begin{remark}[\textbf{tricks: knowledge transfer}]\label{remark:tricks}
    An additional merit of implementing enumeration is the \emph{``knowledge transfer''} capability~\cite{hashemizadeh2020adaptive}.
    We know that increasing the TN-ranks would decrease \emph{monotonically} the value of the objective function in many learning tasks.
    Instead of evaluating each structure explicitly, it thus inspires us to accelerate the structure evaluation by reusing in enumeration the knowledge of the well-optimized core tensors and their corresponding objective.
    In particular, the acceleration by the knowledge transfer trick is two-fold: one is to use the well-optimized core tensors associated with the lower-rank structures to be the initialization for the ones with higher-rank structures, as in~\citet{hashemizadeh2020adaptive}; the other is to employ the \emph{objective estimation}, in which we apply linear interpolation to predicting the objective in the evaluation phase in place of the explicit calculation
    (see Appendix~\ref{apd:sec:lossEstimation}).
    Although the objective estimation would be of no precision, it helps in the first $1\sim{}2$ iterations of TnALE (with a large radius) as initialization for quickly finding a reasonable structure candidate.
\end{remark}

\begin{remark}[\textbf{computational complexity}]
    TnALE is a \emph{meta}-algorithm for TN-SS, in which the inner minimization of~\eqref{eq:basicModel} can be solved by any practitioner-appointed algorithms.
    Therefore, the computational complexity of TnALE is mainly affected by the number of evaluations.
    Suppose the searching problem of order-$N$ with the TN-ranks of dimension $K$.
    Furthermore, suppose each entry $r_i,\,i\in[K]$ of $\mathbf{r}$ is enumerated in $I$ values in the neighborhood, \textit{e.g.}, the interval $[r_i-I/2,r_i+I/2]$, and $D$ times of the ``round-trip'' updates.
    In this setting, for the TN-RS problem, TnALE requires $O(DKI)$ evaluations in one neighborhood; 
    for TN-PS, it requires $O(DKI+DN^2/2)$ evaluations since the neighborhood of $G$ in TN-PS contains $(N-1)N/2$ elements~\cite{li2022permutation} in general;
    last for TN-TS, $O(DN^2I)$ evaluations are required.
    In practice, the value of $I$ is typically set to $3$ or $5$ and $D=1$.
\end{remark}

In summary, the evaluation complexity of TnALE grows polynomially with the TN-order and the dimension of the TN-ranks.
In the next section, we prove that such the number of evaluations is theoretically sufficient in TnALE for achieving quick convergence for TN-SS.

\begin{figure}
    \centering
    \includegraphics[width=1\columnwidth]{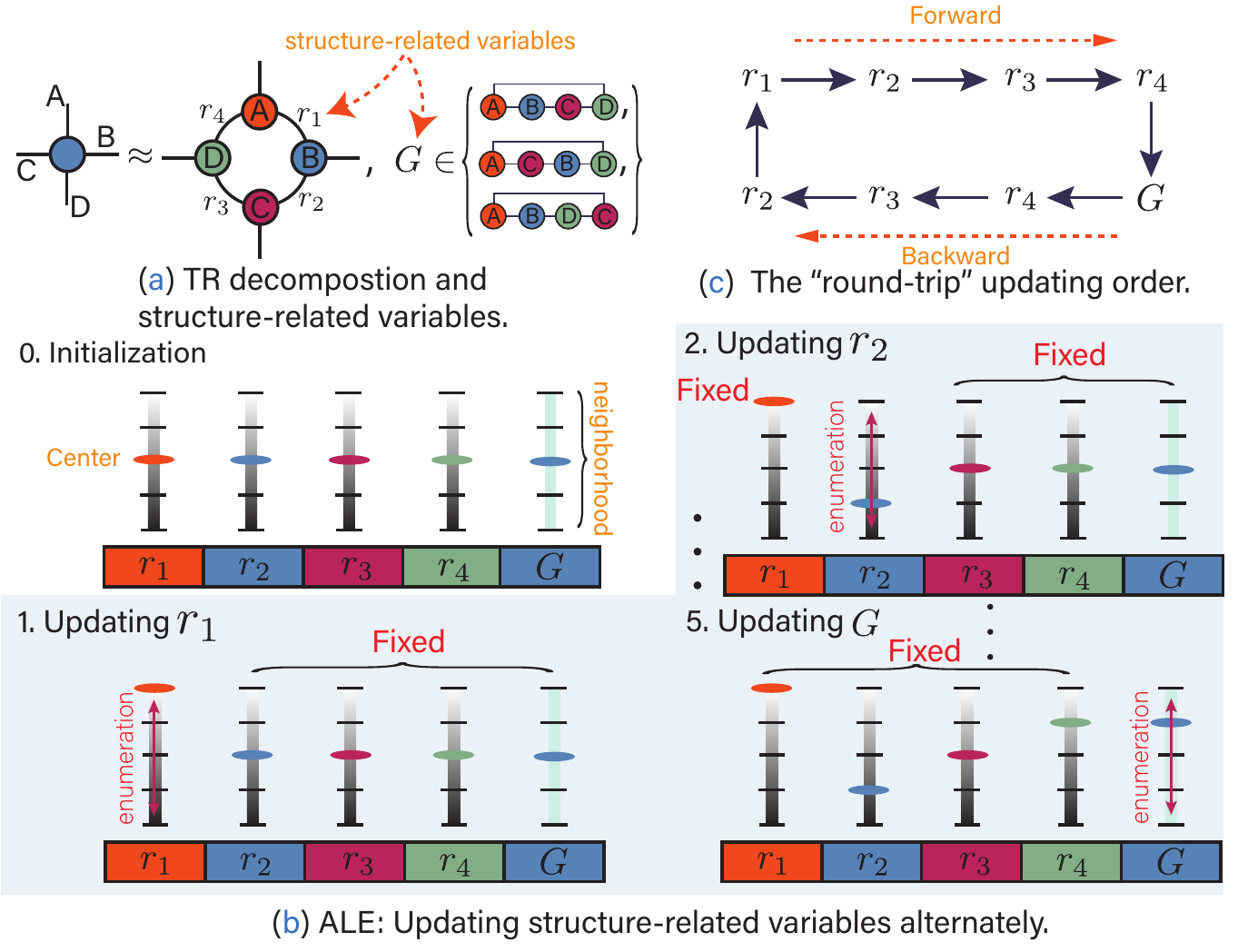}
    \vspace{-0.4cm}
    \caption{Illustration of \emph{alternating local enumeration (ALE)} for TN-PS of TR decomposition.
    The search for TN-RS and TN-TS is similar, except that the ranges of the TN-ranks $r_{\{1,2,3,4\}}$ and the graph $G$ need to be adjusted.
    }
    \label{fig:ALS}
    \vspace{-0.4cm}
\end{figure}

\section{Theoretical Results}
\label{sec:analysis}
In this section, we first analyze the descent steps for both TNLS and TnALE, proving that using the ``local search''~\cite{li2022permutation} scheme the algorithms would achieve a linear convergence rate up to a constant if the objective is sufficiently ``convex'' in the \emph{discrete} domain.
Following this, we analyze the evaluation efficiency for TNLS and TnALE, showing that the required number of evaluations in TNLS would grow exponentially with the dimension of the search space, while it can be ideally reduced to be a linear growth in TnALE if the neighborhood is \emph{low-rank}.
The proofs in this section are given in Appendix~\ref{apd:sec:proofs}.

\subsection{Analysis of Descent Steps}\label{sec:descent}
We start the analysis by rewriting~\eqref{eq:basicModel} into a more general form
\begin{equation}
    \min_{\mathbf{x}\in\mathbb{Z}_+^K,p\in\mathbb{P}}f_p(\mathbf{x}):=f\circ{}p(\mathbf{x}),
    \label{eq:generalModel}
\end{equation}
where $f:\mathbb{Z}_{\geq{}0}^L\rightarrow\mathbb{R}_+$ is a general form of the objective function of~\eqref{eq:basicModel}, $\mathbf{x}\in\mathbb{Z}_+^{K}$ corresponds to the TN-ranks $\mathbf{r}$ of~\eqref{eq:basicModel}, and the operator $p:\mathbb{Z}_+^K\rightarrow\mathbb{Z}_{\geq{}0}^L$ and its feasible set $\mathbb{P}$ correspond to the graph variable $G$ and its feasible set $\mathbb{G}$ of~\eqref{eq:basicModel}, respectively.
The specific relationship of $p$ and $G$ is discussed in Appendix~\ref{apd:sec:analysis}.


The framework used in the proof follows from~\citet{golovin2019gradientless} for the zeroth-order convex optimization.
However, due to the discrete essence of~\eqref{eq:generalModel}, we re-establish discrete analogues of the fundamental concepts such as the gradient, strong convexity and smoothness for the analysis, and all the proofs are re-derived non-trivially in the discrete domain.




In doing so, we begin with the definition of the finite gradient~\cite{olver2014introduction}, as the alternative to the classic gradient in the continuous domain.

\begin{definition}[\textbf{finite gradient}]\label{def:finiteGradient}
        For any function $f:\mathbb{Z}_{\geq{}0}^L\rightarrow{}\mathbb{R}$, its finite gradient $\Delta{}f:\mathbb{Z}_{\geq{}0}^L\rightarrow{}\mathbb{R}^{L}$ with respect to $\mathbf{x}\in\mathbb{Z}_{\geq{}0}^L$ is defined as follows:
    \begin{equation}
    \begin{split}
        &\Delta{}f(\mathbf{x})=\\
        &\left[f(\mathbf{x}+\mathbf{e}_1)-f(\mathbf{x}),\ldots,f(\mathbf{x}+\mathbf{e}_L)-f(\mathbf{x})\right]^\top,\label{eq:finiteGradient}
    \end{split}
    \end{equation}
    where $\mathbf{e}_i,\,1\leq{}i\leq{}L$ denote the unit vectors with the $i$-th entry being one and the others being zeros.
\end{definition}

Next, we redefine the convexity and smoothness of the objective with finite gradients.


\begin{definition}[\textbf{$\alpha$-strong convexity with finite gradient}]\label{def:strongConvexity}
        We say $f$ is $\alpha$-strongly convex for $\alpha\geq{}0$ if $f(\mathbf{y})\geq{}f(\mathbf{x})+\left<\Delta{}f(\mathbf{x})-\frac{\alpha}{2}\mathbf{1},\mathbf{y-x}\right>+\frac{\alpha}{2}\Vert{}\mathbf{y-x}\Vert^2$ for all $\mathbf{x,y}\in\mathbb{Z}_{\geq{}0}^L$, where $\mathbf{1}\in\mathbb{R}^L$ denotes the vector with all entries being one.
    We simply say $f$ is convex if it is $\alpha$-strongly convex and $\alpha=0$.
\end{definition}
\begin{definition}[\textbf{$(\beta_1,\beta_2)$-smoothness with finite gradient}]\label{def:smoothness}
    We say $f$ is $(\beta_1,\beta_2)$-smooth for $\beta_1,\beta_2>0$ if
    \begin{enumerate}
        \item $\vert{}f(\mathbf{x})-f(\mathbf{y})\vert\leq{}\beta_1\Vert\mathbf{x-y}\Vert$ for all $\mathbf{x,y}\in\mathbb{Z}_{\geq{}0}^L$;
        \item The function $l(\mathbf{x}):=\frac{\beta_2}{2}\Vert{}\mathbf{x}\Vert^2-f(\mathbf{x})$ is convex.  
    \end{enumerate}
\end{definition}
We remark that Definition~\ref{def:strongConvexity} and~\ref{def:smoothness} are partially different from the ones used in~\citet{golovin2019gradientless} or other literature for convex analysis.
Particularly in Definition~\ref{def:smoothness}, the smoothness is defined by additionally taking the Lipschitz continuity (corresponding to \emph{Item 1}) into account, which controls the changing rate of $f$, while \textit{Item 2} in Definition~\ref{def:smoothness} controls the changing rate of the finite \emph{gradient} of $f$.
See Lemma~\ref{apd:thm:smooth2Item} in Appendix for the discussion.
With the new definitions, we next give the main assumptions used in the results.
\begin{assumption}\label{def:assumption}
Assume that  $f:\mathbb{Z}_{\geq{}0}^L\rightarrow{}\mathbb{R}_+$ of~\eqref{eq:generalModel} is $\alpha$-strongly convex, $(\beta_1,\beta_2)$-smooth, and its minimum, denoted $(p^*,\mathbf{x}^*)=\arg\min_{p,\mathbf{x}}f\circ{}p(\mathbf{x})$, satisfies $\Vert\Delta{}f_{p^*}(\mathbf{x}^*)-\frac{\beta_2}{2}\mathbf{1}\Vert\leq{}\gamma$ where $0\leq\gamma<\alpha\leq{}\beta_1\leq\beta_2\leq{}1$.
\end{assumption}

In Assumption~\ref{def:assumption}, the inequality $\Vert\Delta{}f_{p^*}(\mathbf{x}^*)-\frac{\beta_2}{2}\mathbf{1}\Vert\leq{}\gamma$ implies that, up to a (small) bias $\frac{\beta_2}{2}\mathbf{1}$, the finite gradient at $(p^*,\mathbf{x}^*)$ should be sufficiently small, which can be analogous to the ``gradient-equal-zero''~\cite{boyd2004convex} property of the stationary points for convex functions in the continuous domain.
The upper bound ``$1$'' is arbitrarily chosen just for simplifying the calculation.

Next, we reveal that the local-search scheme, used in both TNLS and TnALE, can achieve the linear convergence rate up to a constant.
We first focus on TN-RS and TN-TS to simplify the problem, where $p^*$ is assumed to be ~\emph{known} beforehand.
After that, we discuss TN-PS, showing that the searchable $p$ would make the convergence more difficult.


\begin{theorem}[\textbf{convergence rate when $p^*$ is known}]\label{thm:rateFixedP}
    Suppose Assumption~\ref{def:assumption} is satisfied, the operator $p$ of~\eqref{eq:generalModel} is fixed to be $p^*$, and $0\leq{}\theta\leq{}1$.
    Then, for any $\mathbf{x}$ with $\Vert{}\mathbf{x-x}^*\Vert_\infty\leq{}c$, we can find a neighborhood $B_\infty(\mathbf{x},r_\mathbf{x})$ where $r_\mathbf{x}\geq\theta{}c+\frac{1}{2}$, such that there exists an element $\mathbf{y}\in{}B_\infty(\mathbf{x},r_\mathbf{x})$ satisfying
    \begin{equation}
        \begin{split}
            f_{p^*}(\mathbf{y})-f_{p^*}(\mathbf{x}^*)&\leq{}(1-\theta)(f_{p^*}(\mathbf{x})-f_{p^*}(\mathbf{x}^*))+\frac{7}{8}K.\label{eq:rateFixedP}
        \end{split}
    \end{equation}
\end{theorem}
Proving Theorem~\ref{thm:rateFixedP} requires the following lemma, which can be understood as the \emph{discrete} version of the convex-combination inequality of convex functions.

\begin{lemma}[\textbf{convex combination in the discrete domain}]\label{thm:convexComb}
    Suppose $\mathbf{q}=\theta{}\mathbf{x}+(1-\theta)\mathbf{y},\,\forall{}\mathbf{x,y}\in\mathbb{Z}_{\geq{}0}^L,\,\theta\in{}[0,1]$, and there is $\hat{\mathbf{q}}\in{}\mathbb{Z}_{\geq{}0}^L$ following $\mathbf{\Lambda}=\mathbf{q}-\hat{\mathbf{q}}$.
    If $f$ is $\alpha$-strongly convex, then
    \begin{equation}
        \theta{}f(\mathbf{x})+(1-\theta)f(\mathbf{y})\geq{}f(\hat{\mathbf{q}})+\left<\Delta{}f(\hat{\mathbf{q}})-\frac{\alpha}{2}\mathbf{1},\mathbf{\Lambda}\right>+\frac{\alpha}{2}\Vert\mathbf{\Lambda}\Vert^2.\label{eq:convexComb}
    \end{equation}
\end{lemma}
Note that the inequality~\eqref{eq:convexComb} would be the same as the crucial inequality $\theta{}f(\mathbf{x})+(1-\theta)f(\mathbf{y})\geq{}f(\mathbf{q})$ in convex analysis if $\mathbf{\Lambda}=0$.
However, due to $\mathbf{q}\notin\mathbb{Z}_{\geq{}0}^L$ in general, the non-zero $\mathbf{\Lambda}$ is inevitable in the proof, yielding the essential difference from the convex analysis in the continuous domain.
As a consequence, the inequality~\eqref{eq:rateFixedP} shows that the convergence rate is formally close to being linear, but the constant $(7/8)K$ appears on the right-hand side \emph{dampening} the search efficiency.

Suppose the search trajectory $\{f_{p^*}(\mathbf{x}_n)\}_{n=0}^\infty$, of which the starting point $\mathbf{x}_0\in\mathbb{Z}_+^K$ is randomly chosen and $\mathbf{x}_n$ for $n>0$ are determined by the vector $\mathbf{y}$ in Theorem~\ref{thm:rateFixedP}.
As an important corollary, it can be easily proved that $\{f_{p^*}(\mathbf{x}_n)\}_{n=0}^\infty$ converges to $f_{p^*}(\mathbf{x}^*)$ up to a constant if $\Omega(1/K)\leq\theta\leq{}1$.
A rigorous proof for the convergence guarantee can be found in Appendix~\ref{apd:thm:guarantee}. 


\begin{remark}[\textbf{Finding $p^*$ makes the convergence slower}.]
As aforementioned, the non-zero $\Vert{}\mathbf{\Lambda}\Vert_\infty$ decreases the search efficiency due to the additional constant shown in~\eqref{eq:rateFixedP}.
It is known from the proof of Theorem~\ref{thm:rateFixedP} that the constant is derived from the tight bound $\Vert{}\mathbf{\Lambda}\Vert_\infty\leq{}1/2$, following the rounding operation.
However, once the $p$ in~\eqref{eq:generalModel} is searchable as well, $\Vert{}\mathbf{\Lambda}\Vert_\infty$ would turn larger, dampening the convergence more seriously.
To verify this, 
suppose $\mathbf{q}=\theta{}p^*(\mathbf{x^*})+(1-\theta)p_\mathbf{x}(\mathbf{x})$ to be the convex combination between $(p^*,\mathbf{x}^*)$ and any point $(p_\mathbf{x},\mathbf{x})$.
It is known from the proof that, for decreasing the objective, $\hat{\mathbf{q}}$ should satisfy 
$\hat{\mathbf{q}}\in{}B_\infty(\mathbf{q},\Vert\mathbf{\Lambda}\Vert_\infty)\cap{}\mathbb{B}(p_\mathbf{x},\mathbf{x})$, where 
$\mathbb{B}(p_\mathbf{x},\mathbf{x}):=\{\mathbf{z}=\bar{p}(\bar{\mathbf{x}}):\bar{p}\in{}B(p_\mathbf{x}),\bar{\mathbf{x}}\in{}B_\infty(\mathbf{x},r_\mathbf{x})\}$ for some $r_\mathbf{x}>0$ and $B(p_\mathbf{x})$ denotes the neighborhood of $p_\mathbf{x}$ chosen in the algorithm.
We thus see that, for the existence of $\hat{\mathbf{q}}$, the intersection $B_\infty(\mathbf{q},\Vert\mathbf{\Lambda}\Vert_\infty)\cap{}\mathbb{B}(p_\mathbf{x},\mathbf{x})$ should be non-empty.
Following this, it satisfies $\Vert{}\mathbf{\Lambda}\Vert_\infty\geq\min_{\bar{\mathbf{q}}\in{}\mathbb{B}(p_\mathbf{x},\mathbf{x})}\Vert{}\mathbf{q}-\bar{\mathbf{x}}\Vert_\infty\geq{}\min_{\hat{\mathbf{q}}\in\mathbb{Z}_{\geq{}0}^L}\Vert{}\mathbf{q}-\hat{\mathbf{q}}\Vert_\infty=1/2$ in the worst case.
In this case, the larger value of $\Vert{}\mathbf{\Lambda}\Vert_\infty$ means a larger damping term appearing on the right-hand side of~\eqref{eq:rateFixedP}.
\end{remark}

\subsection{Evaluation Efficiency}\label{sec:samplingEfficiency}
Note that a premise for achieving the closely linear convergence rate by TNLS and TnALE is that the expected $\mathbf{y}\in{}B_\infty(\mathbf{x},r_\mathbf{x})$ in Theorem~\ref{thm:rateFixedP} is reachable, meaning that the algorithms \emph{should} find the $\mathbf{y}$ out in each neighborhood.
In the rest of the section, we show that TNLS is required to cost $\Omega(2^K)$ samples in each neighborhood for stably reaching the $\mathbf{y}$, while $O(KIR)$ samples are ideally sufficient for TnALE. Here $K$ denotes the dimension of the search space, $I\in\mathbb{Z}_+$ indicates an integer related to the radius of the neighborhood and $R\in\mathbb{Z}_+$ reflects the \emph{low-rankness of the optimization landscape} in neighborhoods.

We first give the proposition for TNLS as follows.

\begin{proposition}[\textbf{curse of dimensionality for TNLS}]\label{thm:TNLSSammpling}
Let the assumptions in Theorem~\ref{thm:rateFixedP} be satisfied.
Furthermore, assume that $\mathbf{x}^*$ is sufficiently smaller (or larger) than $\mathbf{x}$ entry-wisely, except for a constant number of  entries.
Then the probability of achieving a suitable $\mathbf{y}$ as mentioned in Theorem~\ref{thm:rateFixedP} by uniformly randomly sampling in $B_\infty(\mathbf{x},r_\mathbf{x})$ with $r_\mathbf{x}\geq{}\theta{}c+\frac{1}{2}$ equals $O(2^{-K})$.
\end{proposition}
Note that the additional assumption in Prop.~\ref{thm:TNLSSammpling} is commonly satisfied in practice when the searched TN-ranks are initialized uniformly with large (or small) values.
It is also easily known from Prop.~\ref{thm:TNLSSammpling} that $\Omega(2^K/\epsilon)$ samples are required for achieving the probability $Pr\geq{}\epsilon$ for reaching the $\mathbf{y}$ in the neighborhood.


\begin{remark}
The intuition of Prop.~\ref{thm:TNLSSammpling} is as follows.
Suppose $\mathbf{x}^*$ is entry-wisely smaller than $\mathbf{x}$ without loss of generality, then the objective would be decreased only if most of the entries of $\mathbf{x}$ are updated to be smaller values, \textit{i.e.}, getting closer to $\mathbf{x}^*$.
However, by random sampling, the probability of decreasing most entries of $\mathbf{x}$ would turn smaller exponentially with increasing the dimension $K$, yielding the curse of dimensionality for TNLS.
\end{remark}

In contrast to TNLS, TnALE essentially resolves the curse of dimensionality by leveraging the landscape's low-rank structure.
To verify this, given $(p,\mathbf{x})$, we formulate first the neighborhood $B(p)\times{}B_\infty(\mathbf{x},r_\mathbf{x})$ as a $(K+1)$-th order tensor $\mathcal{B}\in\mathbb{R}^{I_1\times{}I_2\times{}\cdots\times{}I_{K+1}}$.
Here $I_k=2\times{}\lceil{}r_\mathbf{x}\rceil+1$ for $1\leq{}k\leq{}K$ and $I_{K+1}=\vert{}B(p)\vert$.
The $(i_1,i_2,\ldots,i_{K+1})$-th entry of $\mathcal{B}$, written $\mathcal{B}(\mathbf{i})$ with $\mathbf{i}= [i_1,i_2,\ldots,i_{K+1}]^\top$, satisfies 
\begin{equation}
    \mathcal{B}(\mathbf{i})=1/f_{p_{i_{K+1}}}(\mathbf{x}+\mathbf{i}(:K)-(\lceil{}r_\mathbf{x}\rceil+1)),\label{eq:B2B}
\end{equation}
where $\mathbf{i}(:K)$ denotes the $K$-dimensional vector consisting of the first $K$ entries of $\mathbf{i}$, and $p_{i_{K+1}}$ denotes the $i_{K+1}$-th element of $B(p)$ in any ordering fashion.
We remark that the inverse $1/f(\mathbf{x})$ is always valid due to the assumption $f_p(\mathbf{x})>0$ in~\eqref{eq:generalModel} for all $\mathbf{x}$ and $p$.
We also remark that the equation~\eqref{eq:B2B} maps uniquely each element of $B(p)\times{}B_\infty(\mathbf{x},r_\mathbf{x})$ onto the entries of $\mathcal{B}$.

\begin{figure}
    \centering
    \includegraphics[width=0.9\columnwidth]{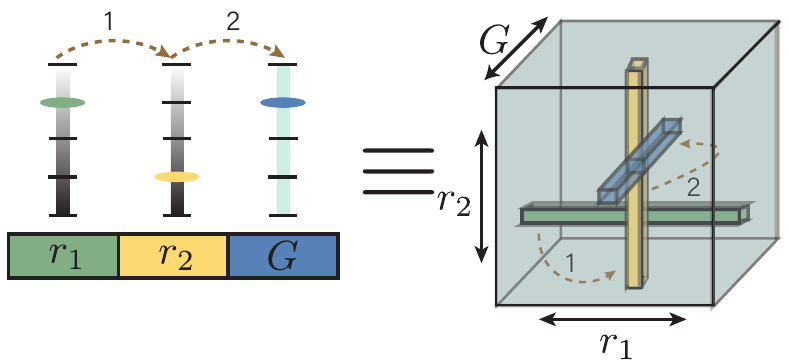}
    \vspace{-0.4cm}
    \caption{The relationship between alternating local enumeration (ALE) and TT-cross approximation~\cite{oseledets2010tt,sozykin2022ttopt}. As shown in the figure, enumerating structure-related variables alternately is equivalent to sampling fibers of a tensor along each mode. The yellow arrows indicate the alternation of variables from $r_1$ to $r_2$ and then to $G$, respectively. }
    \label{fig:B2B}
\end{figure}

By the tensor $\mathcal{B}$, we can find that the proposed \emph{alternating local enumeration (ALE)} is strongly related to TT-cross~\cite{oseledets2010tt} and TTOpt~\cite{sozykin2022ttopt}.
As demonstrated in Figure~\ref{fig:B2B}, the enumeration for each variable is equivalent to drawing a fiber of $\mathcal{B}$ as in TT-cross or TTOpt with the TT-ranks being \emph{ones}.
Such a relationship helps us reveal the evaluation advantage of TnALE.
Specifically, let $\mathbb{B}:=B(p)\times{}B_\infty(\mathbf{x},r_\mathbf{x})$ and $f_\mathbb{B}^*:=\min_{(p_y,\mathbf{y})\in{}\mathbb{B}}f_{p_y}(\mathbf{y})$ for notational simplicity, then we have the following proposition.
\begin{proposition}[\textbf{evaluation efficiency for TnALE}]\label{thm:TnALESampling}
    Let $\mathcal{B}\in\mathbb{R}^{I_1\times{}I_2\times{}\cdots\times{}I_{K+1}}$ be the tensor of order-$(K+1)$ constructed as Eq.~\eqref{eq:B2B} with $I_1=I_2=\cdots=I_{K+1}=I$ for simplicity.
    Then, there exists its TT-cross approximation~\cite{oseledets2010tt} of rank-$R$\footnote{Here all elements of the TT-ranks equal $R$ for simplicity.}, denoted $\hat{\mathcal{B}}$,
    such that  $f_\mathbb{B}^*=f_{p_{j_{K+1}}}\left(\mathbf{x}+\mathbf{j}(:K)-(\lceil{}r_\mathbf{x}\rceil+1)\right)$ with $\mathbf{j}=\arg\max_{\mathbf{i}}\hat{\mathcal{B}}(\mathbf{i})$ holds, provided that
    \begin{equation}
 f_\mathbb{B}^*\leq{}f_{p_z}(\mathbf{z})/\left(1+2\frac{(4R)^{\lceil{}\log_2{}K\rceil}-1}{4R-1}(R+1)^2\xi{}f_{p_z}(\mathbf{z})\right)\label{eq:y*}
    \end{equation}
    for all $(p_z,\mathbf{z})\in{}\mathbb{B}$ and $f_{p_z}(\mathbf{z})\neq{}f_\mathbb{B}^*$.
    Here, $\xi$ denotes the error between $\mathcal{B}$ and its best approximation of TT-ranks $R$ in terms of $\Vert\,\cdot\,\Vert_\infty$.
    Note that the inequality~\eqref{eq:y*} holds trivially if $\mathcal{B}$ is exactly of the TT topology of rank-$R$, and~\citet{oseledets2010tt} shows that the $f_\mathbb{B}^*$ can be recovered from $O(KIR)$ entries from $\mathcal{B}$.
\end{proposition}
Prop.~\ref{thm:TnALESampling} is a natural corollary of Theorem 2 in~\citet{osinsky2019tensor}.
It implies that the desired $\mathbf{y}$ (corresponding to $f_\mathbb{B}^*$ in Prop.~\ref{thm:TnALESampling}) in Theorem~\ref{thm:rateFixedP} is reachable by only $O(KIR)$ samples once $\mathcal{B}$ is of TT with rank-$R$.
Even though $\mathcal{B}$ is not low-rank, the $\mathbf{y}$ can still be located if the inequality~\eqref{eq:y*} is satisfied.

Prop.~\ref{thm:TnALESampling} shows the $O(KIR)$ evaluation advantage compared with TNLS that requires $\Omega(2^K/\epsilon)$ evaluations in the neighborhoods, but it remains open to prove the low-rankness of the optimization landscape in the TN-SS tasks.
We empirically verify this with five synthetic tensors of order four.
We calculate their complete optimization landscape associated with the $l_2$ loss, observing that the multidimensional landscape is indeed low-rank under all possible unfoldings~(see Figure~\ref{apd:fig:landscape} in Appendix).
We thus conjecture that in practice the low-rank structure of the landscape should be preserved, at least in neighborhoods.
In the next section, the evaluation advantage by TnALE will be empirically verified with both synthetic and real-world data.

\section{Experimental Results}
In this section, we present numerical results to verify the superiority of TnALE in terms of evaluation cost.
Due to the page limit, the experimental settings will be presented at the minimum level of clarity.
Additional details are given in Appendix~\ref{apd:sec:experiment}.

\subsection{Synthetic Data}
First of all, we assess the superiority of TnALE by solving the TN-PS problem, in which both the optimal TN-ranks and permutations of synthetic tensors are searched for the tensor decomposition task.

In the experiment, we \emph{re-use} from~\citet{li2022permutation} the synthetic tensors, which are randomly generated in the topologies including TR (order-$8$), PEPS~(order-6,~\citealp{verstraete2004renormalization}), hierarchical Tucker (HT of order-$6$,~\citealp{hackbusch2009new}), and MERA~(order-8,~\citealp{cincio2008multiscale}).
Additionally, we also consider the \emph{tensor wheel} model~(TW of order-$5$,~\citealp{wutensor}).
Since the mode dimension is typically irrelevant to the difficulty of TN-SS, we set them to equalling $3$ in all tensors for simplicity.
For each topology, four tensors~(\textbf{A, B, C, D}) are generated, where the TN-ranks and permutations are randomly selected and  remained unknown.
The goal of this experiment is to compare different approaches to identifying the TN-ranks and permutations for each tensor, meaning that the conditions RSE$\leq{}10^{-4}$ and the \textit{Eff.}$\geq{}1$ are satisfied\footnote{RSE means the relative squared error, and the \textit{Eff.} index~\cite{li2020evolutionary} denotes the ratio of the parameter number of TNs between the searched structure and the one used in data generation. \textit{Eff.}$\geq{}1$ implies that the algorithm finds a TN structure identical or more compact than the one used in data generation.}.
Otherwise, we say the approach fails in the experiment.

We implement three algorithms, TNGA~\cite{li2020evolutionary}, TNLS and TnALE~(ours).
Note that in TNGA both the TN ranks rank permutations are encoded as chromosomes, as implemented in the work by~\citet{li2022permutation}.
In the subsequent experiments, the  the encoding scheme of TNGA would be also properly adjusted for handling different sub-problems of TN-SS.
For a fair comparison, the three approaches use the same objective and solver for the inner minimization of~\eqref{eq:basicModel}.
Furthermore, TNLS and TnALE are initialized with the same TN structures.
The rest of the experimental settings remain as~\citet{li2022permutation}.

\begin{table}[tb]
	\centering
	\caption{
	Number of evaluations, denoted ``\emph{\#Eva.}'', for the rank and permutation identification, where the symbol ``-'' in the table means the failure of the approach.
	}
 \label{Tab:TNPS}
	\vskip 0.15in
	\begin{threeparttable}\footnotesize
		\setlength{\tabcolsep}{1.5mm}{   	
			\begin{tabular}{llcccc}
				\toprule
	\multirow{2}[0]{*}{\textbf{Topology}}&\multirow{2}[0]{*}{\textbf{Methods}}&\multicolumn{4}{c}{\textbf{Data} -- \emph{\#Eva.~$\downarrow$} }\\
				\cmidrule{3-6}&
				& \textbf{A} & \textbf{B} & \textbf{C}  & \textbf{D} \\
				\midrule
				
				


                \multirow{3}[0]{*}{\textbf{TR}}
				&TNGA&2850&2250&3900&1950
				\\
				&TNLS&1020&960&1320&780
				\\
				&Ours&\textbf{231}&\textbf{308}&\textbf{308}&\textbf{231}
				\\
                \cmidrule{2-6}

                \multirow{3}[0]{*}{\textbf{PEPS}}
				&TNGA&1560&-&840&1080
				\\
				&TNLS&781&781&421&661
				\\
				&Ours&\textbf{407}&\textbf{465}&\textbf{233}&\textbf{175}
				\\
                \cmidrule{2-6}

                \multirow{3}[0]{*}{\textbf{HT}}
				&TNGA&960&1320&840&1080
				\\
				&TNLS&841&841&781&721
				\\
				&Ours&\textbf{211}&\textbf{281}&\textbf{211}&\textbf{211}
				\\
                \cmidrule{2-6}

                \multirow{3}[0]{*}{\textbf{MERA}}
				&TNGA&-&960&2800&3240
				\\
				&TNLS&1561&841&1441&721
				\\
				&Ours&\textbf{1450}&\textbf{484}&\textbf{323}&\textbf{323}
				\\
                \cmidrule{2-6}

                \multirow{3}[0]{*}{\textbf{TW}}
				&TNGA+&1920&1440&600&720
				\\
				&TNLS&661&601&601&481
				\\
				&Ours&\textbf{285}&\textbf{143}&\textbf{285}&\textbf{214}
				\\

				\bottomrule
				
			\end{tabular}
		}
	\end{threeparttable}
	\vskip -0.1in
\end{table}

The experimental results are shown in Table~\ref{Tab:TNPS}.
We see that both TNLS and TnALE (ours) successfully identify the ranks and permutations for all tensors.
Furthermore, TnALE requires \emph{significantly} fewer evaluations than both TNGA and TNLS.
Figure~\ref{fig:TNPS:AverageCurves} further illustrates the change of the objective values (in $\log$, averaged) versus the number of evaluations in TNLS and TnALE.
The result confirms the consistency of the evaluation advantage of TnALE compared with TNLS\footnote{Note that the curves for MERA exhibit the opposite pattern compared to others due to the ``local-convergence'' issue. This phenomenon is further discussed in Appendix~\ref{apd:sec:experiment}.}.

\begin{figure}[tp!]
    \centering
    \includegraphics[width=1\columnwidth]{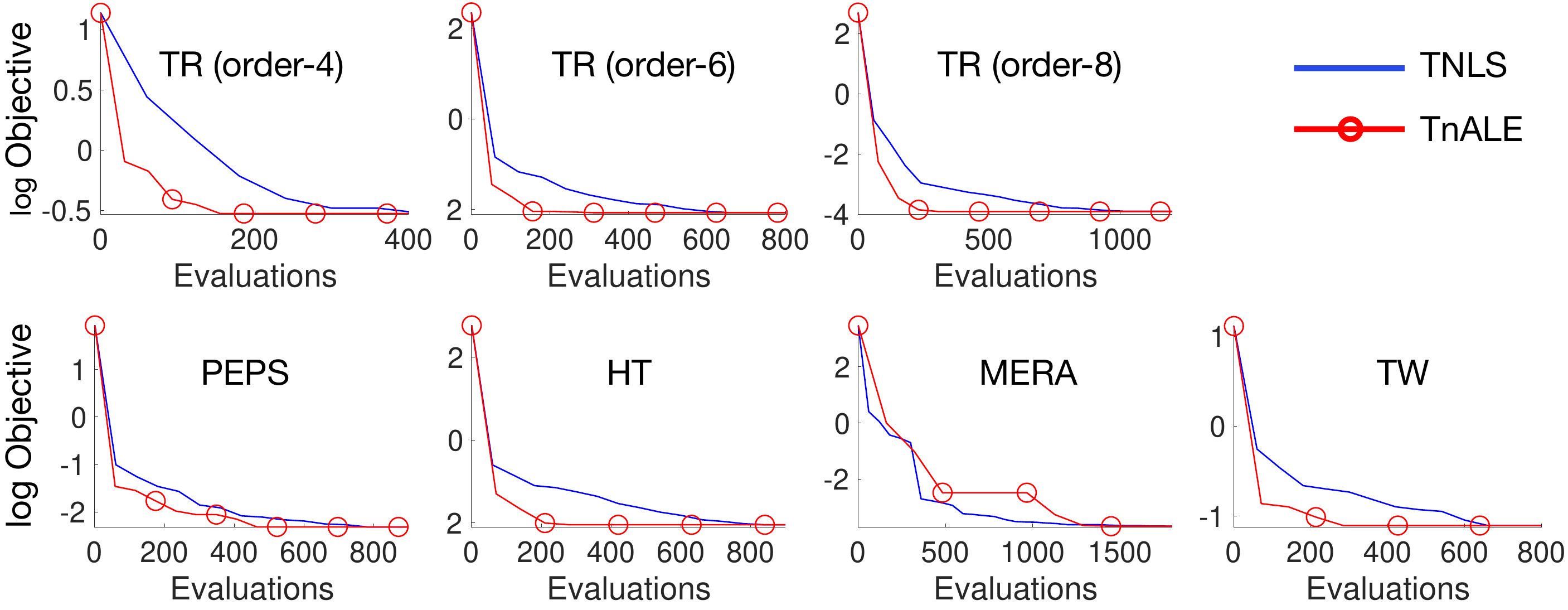}
    \vspace{-0.4cm}
    \caption{Averaged objective (in the $\log$ form) with varying the number of evaluations.}
    \label{fig:TNPS:AverageCurves}
\end{figure}

Next, we evaluate the performance of TnALE for solving the classic rank selection problem, \textit{i.e.}, TN-RS, within the TR decomposition task.
To be specific, we randomly generate synthetic TR-tensors of order $8$, and consider two configurations:
``\emph{lower-ranks}'' and ``\emph{higher-ranks}''.
In the ``lower-ranks'' class, the TN-ranks
are randomly chosen in the interval $[1,4]$, while in the ``higher-ranks'' class the selection interval is lifted to $[5,8]$, so that the ranks would be larger than the tensors' mode dimension (which equals $3$ in this experiment).
This situation commonly happens in practice for high-order TNs.
For comparison, we implement various rank-adaptive TR decomposition methods, including TR-SVD and TR-BALS~\cite{zhao2016tensor}, TR-LM~\cite{mickelin2020algorithms}, and TRAR~\cite{sedighin2021adaptive}.
In addition, the TTOpt algorithm~\cite{sozykin2022ttopt} with ranks~\footnote{Here the ranks are the tuning parameters in the TTOpt algorithm, rather than the targeted TN structure.}, denoted $R$, equaling $\{1,2\}$ is also employed as a baseline.

The experimental results are shown in Table~\ref{Tab:TNRS}.
We see that most of the methods can successfully identify the TN-ranks (implied by \textit{Eff.}$\geq{}1$) in the ``lower-ranks'' class, but in the ``higher-ranks'' class only the methods at the bottom of the table manage to find the correct ranks.
Furthermore, TnALE costs the fewest evaluations on average compared with TNGA, TNLS and TTOpt.

\begin{table}[tb]
\caption{Experimental results of TN-RS (rank-selection) in $8$th-order TR topology. The columns of ``lower-ranks'' and ``higher-ranks'' indicate two experimental settings by which the TN-ranks are randomly selected.
The \emph{Eff.} and [\#\textit{Eva.}] values are averaged in five tensors.
}
\label{Tab:TNRS}
\vskip 0.15in
\begin{center}
\begin{small}
    \begin{tabular}{lcc}
    \toprule
    \textbf{Methods} & \textbf{lower-ranks} & \textbf{higher-ranks} \\
    \midrule
          & \textit{Eff.}$\uparrow$ & \textit{Eff.}$\uparrow$ \\
    \textbf{TR-SVD} & 0.65$\pm$0.46 & 0.13$\pm$0.20 \\
    \textbf{TR-BALS} & 1.15$\pm$0.14 & 0.19$\pm$0.22 \\
    \textbf{TR-LM} & 1.15$\pm$0.14 & 0.15$\pm$0.02 \\
    \textbf{TRAR} & 0.55$\pm$0.10 & 0.63$\pm$0.06 \\
    \midrule
          & \textit{Eff.$\uparrow$~[\#Eva.$\downarrow$]} & \textit{Eff.$\uparrow$~[\#Eva.$\downarrow$]} \\
    \textbf{TNGA} &  1.08$\pm$0.06~[552]   & 1.00$\pm$0.00~[900] \\
    \textbf{TNLS} &   1.08$\pm$0.06~[492]    & 1.00$\pm$0.00~[588] \\
    \textbf{TTOpt}~($R=1$) & 1.08$\pm$0.06~[104] & 1.00$\pm$0.00~[178] \\
    \textbf{TTOpt}~($R=2$) & 1.02$\pm$0.02~[314] & 1.00$\pm$0.00~[752] \\
    \textbf{Ours} & 1.08$\pm$0.06~[\textbf{80}] & 1.00$\pm$0.00~[\textbf{119}]\\    
    \bottomrule
    \end{tabular}%
\end{small}
\end{center}
\vskip -0.1in
\end{table}

\subsection{Real-World Data}
We apply now the proposed method to real-world data to compress the learnable parameters of the tensorial Gaussian process  (TGP,~\citealt{izmailov2018scalable}) and to compress natural images.
In TGP compression,
we consider the regression task by TGPs for three datasets, including CCPP~\cite{tufekci2014prediction}, MG~\cite{flake2002efficient}, and Protein~\cite{Dua2019}, and compress the variational mean of the process with the TT/TR decomposition using the same settings as in~\citet{li2022permutation}.
The goal of the experiment is to search for good TN-ranks and permutations, so that fewer parameters can be used to achieve the same mean square error (MSE) in regression.
The experimental results are shown in Table~\ref{tab:GP}.
We can see that TnALE achieves the same compression ratio as TNGA and TNLS but costs \emph{significantly} fewer evaluations than the counterparts in factor  up to $14~(3901/276)$.

\begin{table}[t]
\caption{Number of parameters ($\times{}1000$, $\downarrow$) and MSE (in round brackets) for TGP model compression, where the values in [square brackets] show the number of evaluations required in each method.}
\label{tab:GP}
\vspace{0.3cm}
\centering\scriptsize
	\begin{threeparttable}\centering
			\begin{tabular}{lccc}
				\toprule

			& \textbf{CCPP} & \textbf{MG} & \textbf{Protein} \\
				\midrule
				
			\textbf{TGP}&2.64~(0.06)~[N/A]&3.36~(0.33)~[N/A]&2.88~(0.74)~[N/A]\\
		    \textbf{TNGA}&\textbf{2.24}~(0.06)~[1500]&\textbf{3.01}~(0.33)~[4900]&2.03~(0.74)~[3900]\\
		    \textbf{TNLS}&\textbf{2.24}~(0.06)~[1051]&\textbf{3.01}~(0.33)~[3901]&\textbf{1.88}~(0.74)~[3601]\\
      \textbf{Ours}&\textbf{2.24}~(0.06)~[\textbf{124}]&\textbf{3.01}~(0.33)~[\textbf{276}]&\textbf{1.88}~(0.74)~[\textbf{1053}]\\
				\bottomrule
						   
			\end{tabular}
	\end{threeparttable}
	\vskip -0.1in
	\end{table}

Last, we consider the TN-PS and TN-TS tasks for compressing natural images.
In TN-TS, we search for good TN-ranks and topologies for image compression.
In the experiment, we randomly select four images (\textbf{A, B, C, D}, see Figure~\ref{fig:livecompression}) from the dataset BSD500~\cite{arbelaez2010contour}.
Each image is resized by $256\times{}256$ and then reshaped into an order-$8$ tensor.
For comparison, we also implement the ``Greedy'' method~\cite{hashemizadeh2020adaptive} in the TN-TS task.
Table~\ref{tab:TNTS} shows the results, including the compression ratio, RSE, and the number of evaluations in both tasks. 
We see that TnALE achieves the closest compression ratio and RSE to TNGA and TNLS, but it requires much fewer evaluations.

\section{Concluding Remarks}
\begin{figure}[t]
    \centering
    \begin{minipage}[t]{0.24\linewidth}
        \includegraphics[width=1\textwidth]{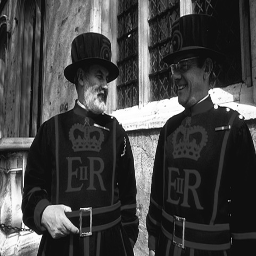}
    \end{minipage}
    \begin{minipage}[t]{0.24\linewidth}
        \includegraphics[width=1\textwidth]{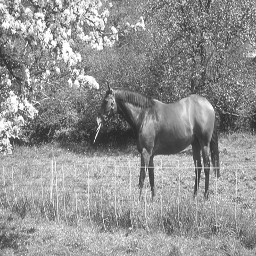}
    \end{minipage}
    \begin{minipage}[t]{0.24\linewidth}
        \includegraphics[width=1\textwidth]{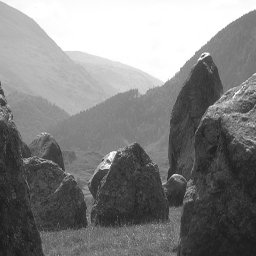}
    \end{minipage}
    \begin{minipage}[t]{0.24\linewidth}
        \includegraphics[width=1\textwidth]{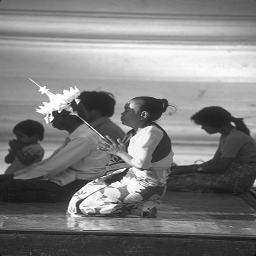}
    \end{minipage}
    \vspace{-0.4cm}
    \caption{Four images in the compression experiment. 
    }
    \label{fig:livecompression}
\end{figure}
\begin{table}[t]
  \centering\scriptsize
  \caption{Results for natural image compression. The underlined values show the best compression ratio achieved in the same RSE.}
  \vspace{0.3cm}
    \begin{tabularx}{\columnwidth}{>{\raggedright\arraybackslash}X >{\raggedright\arraybackslash}X >{\raggedright\arraybackslash}X >{\raggedright\arraybackslash}X >{\raggedright\arraybackslash}X >{\raggedright\arraybackslash}X} 
    \toprule
    \multirow{2}[4]{*}{\textbf{Tasks}} & \multirow{2}[4]{*}{\textbf{Methods}} & \multicolumn{4}{c}{\textbf{Data}~-~compression ratio~($\log,\,\uparrow$) (RSE~$\downarrow$) [\textit{\#Eva.}~$\downarrow$]} \\ 
\cmidrule{3-6}          &       & \multicolumn{1}{c}{\textbf{A}} & \multicolumn{1}{c}{\textbf{B}} & \multicolumn{1}{c}{\textbf{C}} & \multicolumn{1}{c}{\textbf{D}} \\
    \midrule
    \multirow{3}[2]{*}{\textbf{TN-PS}} & \textbf{TNGA} & 1.10~(0.15) [8400] & 1.37~(0.17) [6300] & 1.77~(0.08) [4800] & 1.47~(0.10) [5100] \\[3.5ex]
          & \textbf{TNLS}  & 1.09~(0.16) [1351] & \underline{1.41}~(0.17) [1501] & 1.71~(0.08) [2551] & 1.47~(0.10) [2101] \\[3.5ex]
          & \textbf{Ours} & 1.14~(0.16) [\textbf{647}] & 1.39~(0.17) [\textbf{666}] & \underline{1.80}~(0.08) [\textbf{394}] & \underline{1.49}~(0.10) [\textbf{444}] \\
    \midrule
    \multirow{4}[2]{*}{\textbf{TN-TS}} & \textbf{Greedy} & \multicolumn{1}{c}{0.81~(0.16)} & \multicolumn{1}{c}{0.97~(0.17)} & \multicolumn{1}{c}{1.44~(0.08)} & \multicolumn{1}{c}{0.68~(0.10)} \\[1.5ex]
          & \textbf{TNGA} & 1.16~(0.16) [2100] & \underline{1.48}~(0.17) [1800] & \underline{1.81}~(0.08) [1900] & 1.48~(0.09) [1000] \\[3.5ex]
          & \textbf{TNLS}  & 1.15~(0.16) [1300] & 1.40~(0.17) [1100] & 1.80~(0.08) [1700] & 1.50~(0.10) [1700] \\[3.5ex]
          & \textbf{Ours} & 1.10~(0.15) [\textbf{177}] & 1.46~(0.17) [\textbf{153}] & \underline{1.81}~(0.08) [\textbf{237}] & 1.51~(0.10) [\textbf{246}] \\
    \bottomrule
    \end{tabularx}%
  \label{tab:TNTS}%
\end{table}%

Extensive experimental results demonstrate that the proposed TnALE approach can \emph{greatly} reduce the evaluation cost, up to $10\times$ fewer evaluations, compared with TNLS~\cite{li2022permutation} and other methods for the task of tensor network structure search (TN-SS).
The theoretical results in this paper provide a  rigorous analysis of the convergence rate and the evaluation efficiency for both TNLS and TnALE.

\textbf{Limitation.}
The main limitation of TnALE is the local convergence issue.
In particular, we empirically found in the TN-TS experiment (see Table~\ref{tab:TNTS}) multiple local minima, which are poor in compression ratio, and TnALE can easily drop in them.
Conversely, the methods TNGA~\cite{li2020evolutionary} and TNLS~\cite{li2022permutation} seem to better avoid local minima, owing to their stochastic essence.
Solving this issue will be the direction of our future work.
Furthermore, the identifiability of the proposed method for TN-SS in the presence of \emph{noise} would be also investigated in the future.


\section*{Acknowledgements}
This work was partially supported by JSPS KAKENHI (Grant No. 20H04249, 23H03419).
Chunmei is supported by China Scholarship Council (CSC). 
Part of the computation was carried out at the Riken AIp Deep learning ENvironment (RAIDEN).

\bibliography{references}
\bibliographystyle{icml2023}

\newpage
\appendix
\onecolumn
\section{TnALE: Details for the algorithm}\label{apd:sec:TnALE}
The pseudocode for ALE is given in Alg.~\ref{alg:ALE}.
As discussed in the paper, each structure-related variable is updated alternately by enumeration in the neighborhood.
Based on it, the entire algorithm of TnALE is shown in Alg.~\ref{alg:TnALE}.
Apart from the major iteration (lines 6-8), beforehand, there is an initial phase, where we apply a larger radius $r_1$ and the objective prediction trick to find the candidates in a broader range and a rough fashion.

\subsection{The objective estimation trick.}\label{apd:sec:lossEstimation}
As discussed in Remark~\ref{remark:tricks}, we employ the \emph{objective estimation} by linear interpolation in place of the complete enumeration when updating the TN-ranks.
Particularly in Alg.~\ref{alg:TnALE}, we apply this trick to the initial phase, where we consider a broader range of the neighborhood for roughly finding the structure candidates.

Suppose a rank variable $r$ is required to be enumerated in the range of $r_0-b\leq{}r\leq{}r_0+b$, where $r_0,b\in\mathbb{Z}_+$ present the center and radius of the searched neighborhood, respectively.
The objective estimation trick aims to estimate the minimum of the inner optimization of~\eqref{eq:basicModel} associated with each enumerated TN-structures with limited evaluations.
In the trick, we first evaluate explicitly three TN structures, \textit{i.e.}, $r=\{r_0-b,r_0,r_0+b\}$.
Then, a simple linear regression model is applied to predicting the evaluations of TN-structures in the interval $(r_0,r_0+b)$.
The evaluations for the interval $(r_0-b,r_0)$ will be predicted similarly using the data \textit{w.r.t.} $\{r_0-b,r_0\}$.
We can see from the trick that the inner minimization of~\eqref{eq:basicModel} can be quickly estimated with only three explicit evaluations, no matter how wide the searching range is.
Although the simple linear regression can only give a rough estimation, it is sufficiently helpful for TnALE to find good TN-structures in the initial phase.

\subsection{The neighborhood in the graph space}\label{apd:sec:neighborhood}
In the problem of TN-PS, we follow the idea of~\citet{li2022permutation} to specify the neighborhood for a given graph $G$ in~\eqref{eq:basicModel}.
Similar to Alg.~1 in~\citet{li2022permutation}, we construct the neighborhood by swapping enumerately two vertices of $G$.
Suppose the graph $G$ is of $N$ vertices, we consequently achieve the neighborhood of $G$ of the size $N(N-1)/2$.

\begin{algorithm}[]
  \caption{TnALE: the proposed solver of the optimization~\eqref{eq:basicModel}}
  \label{alg:TnALE}
\begin{algorithmic}[1]
  \STATE {\bfseries Input:}
  A solver for the inner minimization of~\eqref{eq:basicModel}; the rank-related radius: $r_1\geq{}r_2>0$;
  the number of Iterations in the initialization phase: $L_0$;
  the number of Iterations in the searching phase: $L$;
  the number of the round-trips for ALE: $D$;
  \STATE {\bfseries Initialize:}
  Uniformly choose a TN structure $(G,\mathbf{r})$ at random or choose $(G,\mathbf{r})$ with the specified value;
  \FOR{$l=1,\ldots{}L_0$}
    \STATE Update recursively $(G,\mathbf{r})$ using Alg.~\ref{alg:ALE} with the center $(G,\mathbf{r})$, radius $r_1$, $D$ round-trips and the objective estimation trick;
  \ENDFOR
  \FOR{$l=1,\ldots{}L$}
    \STATE Update recursively $(G,\mathbf{r})$ using Alg.~\ref{alg:ALE} with the center $(G,\mathbf{r})$, radius $r_2$ and $D$ round-trips ;
  \ENDFOR
  \STATE {\bfseries Output:} $(G,\mathbf{r})$.
\end{algorithmic}
\end{algorithm}

\begin{algorithm}[]
  \caption{ALE: alternating local enumeration}
  \label{alg:ALE}
\begin{algorithmic}[1]
  \STATE {\bfseries Input:}
  The center of the neighborhood: $(G^{(0)},\mathbf{r}^{(0)})$, where $\mathbf{r}=[r^{(0)}_1,r^{(0)}_2,\ldots,r^{(0)}_K]^\top\in\mathbb{Z}_+^K$;
  the rank-related radius: $r\in\mathbb{Z}_+$; the number of ``round-trips'': $D$;
  \STATE {\bfseries Initialize:} $(G,\mathbf{r})=(G_0,\mathbf{r}_0)$, where $\mathbf{r}=[r_1,r_2,\ldots,r_K]^\top$
  \FOR{$d=1,\ldots,D$} 
      \FOR{$k=1,\ldots,K$}
          \FOR{$i=-r,\ldots,0,\ldots,r$}
            \STATE Copy $(G,\mathbf{r})$ into $(\Bar{G},\bar{\mathbf{r}})$
            \STATE Update $(\Bar{G},\bar{\mathbf{r}})$ by $\bar{r}_k=r_k+i$;
            \STATE Calculate the objective of~\eqref{eq:basicModel} associated to $(\Bar{G},\bar{\mathbf{r}})$;
            \COMMENT{Objective estimation is available.}
            \STATE Store the value of the objective as $h(i)$; 
          \ENDFOR
          \STATE Update $(G,\mathbf{r})$ by $r_k=\arg\min_i{}h(i)$;
      \ENDFOR
      \STATE {Take the neighborhood $B(G)$ according to section~\ref{apd:sec:neighborhood}};
      \FOR{all $G'\in{}B(G)$}
      \STATE Update $(G,\mathbf{r})$ by $G=G'$;
      \STATE Calculate the objective of~\eqref{eq:basicModel} associated to $(G,\mathbf{r})$;
      \STATE Store the value of the objective as $h(G')$;
      \ENDFOR
      \STATE Update $(G,\mathbf{r})$ by $G=\arg\min_{G'}{}h(G')$;
      \FOR{$k=K,K-1,\ldots,2$}
          \FOR{$i=-r,\ldots,0,\ldots,r$}
            \STATE Copy $(G,\mathbf{r})$ into $(\Bar{G},\bar{\mathbf{r}})$
            \STATE Update $(\Bar{G},\bar{\mathbf{r}})$ by $\bar{r}_k=r_k+i$;
            \STATE Calculate the objective of~\eqref{eq:basicModel} associated to $(\Bar{G},\bar{\mathbf{r}})$; \COMMENT{Objective estimation is available.}
            \STATE Store the value of the objective as $h(i)$;
          \ENDFOR
          \STATE Update $(G,\mathbf{r})$ by $r_k=\arg\min_i{}h(i)$;
      \ENDFOR
  \ENDFOR

  \STATE {\bfseries Output:} $(G,\mathbf{r})$.
\end{algorithmic}
\end{algorithm}


\newpage
\section{Theoretical analysis with proofs}\label{apd:sec:proofs}
In this section, we first give a rigorous convergence analysis for the algorithms using the local-sampling scheme.
After that, we compare the evaluation efficiency for TNLS~\cite{li2022permutation} and the new algorithm TnALE.

\subsection{A quick review of tensor network (TN) structure search }
Suppose we have the dataset $D$ and a task-specific loss function $\pi_D:\mathbb{R}^{I_1\times{}I_2\times\cdots\times{}I_N}\rightarrow{}\mathbb{R}_+$ associated to $D$.
The tensor network structure search (TN-SS) problem is to solve the following bi-level optimization problem~\cite{li2022permutation}
\begin{equation}
    \begin{split}
        \min_{(G,\mathbf{r})\in{}\mathbb{G}\times{}\mathbb{F}_G}{}\left(\phi(G,\mathbf{r})+\lambda\cdot\min_{\mathcal{Z}\in{}TNS(G,\mathbf{r})}\pi_D(\mathcal{Z})\right)\label{apd:eq:basicModel}
    \end{split},
\end{equation}
where $G\in\mathbb{G}$ is a graph, which owns $N$ vertices and $K$ edges and corresponds to the TN-topology, $\mathbf{r}\in\mathbb{F}_N\subseteq\mathbb{Z}_+^K$ is a positive integer vector of $K$ dimension corresponding to the TN-ranks, $\phi:\mathbb{G}\times{}\mathbb{Z}_+^K\rightarrow{}\mathbb{R}_+$ represents the function measuring the model complexity of a TN whose structure is modeled by $(G,\mathbf{r})$, and $\lambda>0$ is a tuning parameter.
As expected for TN-SS, solving the problem~\eqref{apd:eq:basicModel} is intuitively to search for a TN structure modeled by $(G,\mathbf{r})$, by which we can give the optimal balance between the complexity and the expressibility of a TN in the task.

We remark that TN-SS can be specified as three sub-problems: \emph{permutation selection (TN-PS,~\citet{li2022permutation}), rank selection (TN-RS) and topology selection (TN-TS,~\citet{li2020evolutionary})}, by specifying the feasible set $\mathbb{G}\times{}\mathbb{F}_G$ of~\eqref{apd:eq:basicModel} into different forms.
Specifically, in TN-PS, we set $\mathbb{F}_G=\mathbb{Z}_+^{K}$ and $\mathbb{G}$ is defined as the isomorphic graphs to a ``template'' $G_0$, so that only the ranks and vertex permutations are searched while the TN-topology is preserved.
In TN-RS, however, we typically restrict the graph $G$ in~\eqref{apd:eq:basicModel}, \textit{i.e.}, $\mathbb{G}=\{G_0\}$ but consider searching for all possible ranks \textit{i.e.}, $\mathbb{F}_G=\mathbb{Z}_+^K$.
In TN-TS, we relax $\mathbb{G}$ to be a set containing all possible simple graphs of order $N$ and the set $\mathbb{F}_N$ can be fixed~\cite{li2020evolutionary} or not~\cite{hashemizadeh2020adaptive}.
It is known from~\cite{Ye2019Tensor,hashemizadeh2020adaptive} that the TN-PS problem with the rank searching can be simplified as a TN-RS problem associated to a ``fully-connected'' TN~\cite{zheng2021fully}.

\subsection{Analysis of descent steps}
\label{apd:sec:analysis}

We start the analysis by rewriting~\eqref{apd:eq:basicModel} into a more general form:
\begin{equation}
    \min_{\mathbf{x}\in\mathbb{Z}_+^K,p\in\mathbb{P}}f_p(\mathbf{x}):=f\circ{}p(\mathbf{x}),
    \label{apd:eq:generalModel}
\end{equation}
where $\circ$ denotes the function composition, $f:\mathbb{Z}_{\geq{}0}^L\rightarrow\mathbb{R}_+$ is a generalization of the objective function of~\eqref{apd:eq:basicModel}, $\mathbf{x}\in\mathbb{Z}_+^{K}$ corresponds to the rank-related variable $\mathbf{r}$ of~\eqref{apd:eq:basicModel}, and the operator $p:\mathbb{Z}_+^K\rightarrow\mathbb{Z}_{\geq{}0}^L$ and its feasible set $\mathbb{P}$ correspond to the topology-related variable $G$ and the set $\mathbb{G}$ of~\eqref{apd:eq:basicModel}, respectively.

The relationship between $p\in\mathbb{P}$ of~\eqref{apd:eq:generalModel} and $G\in\mathbb{G}$ of~\eqref{apd:eq:basicModel} is demonstrated as follows.
Since in~\eqref{apd:eq:basicModel} the entries of $\mathbf{r}$ can be regarded as labels on the edges of $G$, the pair $(G,\mathbf{r})$ can be described as a weighted adjacency matrix of $N\times{}N$.
For example, a $4$-th order tensor ring~(TR,~\citealp{zhao2016tensor}) of the ranks-$\{2,3,4,5\}$ can be described as
\begin{equation}
    (G_1,\mathbf{r})=
    \left(
            \begin{pmatrix}
            0 & 0 & 1 & 1\\
            0 & 0 & 1 & 1\\
            1 & 1 & 0 & 0\\
            1 & 1 & 0 & 0
            \end{pmatrix},
            \begin{pmatrix}
                2\\
                3\\
                4\\
                5
            \end{pmatrix}
            \right)
            \implies
    \mathbf{A}_{1}=\left(
            \begin{matrix}
            0 & 0 & 2 & 3\\
            0 & 0 & 4 & 5\\
            2 & 3 & 0 & 0\\
            4 & 5 & 0 & 0
            \end{matrix}\right),
\end{equation}
or
\begin{equation}
(G_2,\mathbf{r})=
    \left(
            \begin{pmatrix}
            0 & 1 & 0 & 1\\
            1 & 0 & 1 & 0\\
            0 & 1 & 0 & 1\\
            1 & 0 & 1 & 0
            \end{pmatrix},
            \begin{pmatrix}
                2\\
                3\\
                4\\
                5
            \end{pmatrix}
            \right)
            \implies
    \mathbf{A}_{2}=\left(
            \begin{matrix}
            0 & 2 & 0 & 5\\
            2 & 0 & 3 & 0\\
            0 & 3 & 0 & 4\\
            5 & 0 & 4 & 0
            \end{matrix}\right).
\end{equation}
Here $G_1$ and $G_2$ correspond to the TR topology with different permutations of vertices.
In the settings of TN-PS~\cite{li2022permutation}, we can prove that such the relationship is bijective.
The operator $p$ is thus to map the TN-ranks, denoted $\mathbf{x}\in\mathbb{Z}_+^K$ in~\eqref{apd:eq:generalModel}, onto the vectorization of entries in the upper triangle part (except for the diagonal) of the adjacency matrix.
For example,
\begin{equation}
        \begin{split}
        \mathbf{A}_{1}=\left(
            \begin{matrix}
            0 & \color{blue}0 & \color{blue}2 & \color{blue}3\\
            0 & 0 & \color{blue}4 & \color{blue}5\\
            2 & 3 & 0 & \color{blue}0\\
            4 & 5 & 0 & 0
            \end{matrix}\right)
            \iff
            p_1(\mathbf{x})=
            \begin{pmatrix}
                0 & 0 & 0 & 0\\
                1 & 0 & 0 & 0\\
                0 & 1 & 0 & 0\\
                0 & 0 & 1 & 0\\
                0 & 0 & 0 & 1\\
                0 & 0 & 0 & 0
            \end{pmatrix}
            \begin{pmatrix}
                2\\
                3\\
                4\\
                5
            \end{pmatrix}
            =
            \begin{pmatrix}
                0\\
                2\\
                3\\
                4\\
                5\\
                0
            \end{pmatrix}
        \end{split},
    \end{equation}
    and
\begin{equation}
        \begin{split}
        \mathbf{A}_{2}=\left(
            \begin{matrix}
            0 & \color{blue}2 & \color{blue}0 & \color{blue}5\\
            2 & 0 & \color{blue}3 & \color{blue}0\\
            0 & 3 & 0 & \color{blue}4\\
            5 & 0 & 4 & 0
            \end{matrix}\right)
            \iff
            p_2(\mathbf{x})=
            \begin{pmatrix}
                1 & 0 & 0 & 0\\
                0 & 0 & 0 & 0\\
                0 & 0 & 0 & 1\\
                0 & 1 & 0 & 0\\
                0 & 0 & 0 & 0\\
                0 & 0 & 1 & 0
            \end{pmatrix}
            \begin{pmatrix}
                2\\
                3\\
                4\\
                5
            \end{pmatrix}
            =
            \begin{pmatrix}
                2\\
                0\\
                5\\
                3\\
                0\\
                4
            \end{pmatrix}
        \end{split}.
    \end{equation}
It is shown that $p$ is essentially an operator produced by the permutation padding with several rows of zeros according to $G$.

The convergence analysis of this work is mainly inspired by~\citet{golovin2019gradientless}, which establishes a convex framework for the gradient-less optimization algorithms in the real domain.
The challenge is, the TN-SS problem is essentially discrete, so that many well-developed tools, such as convexity and smoothness, for convergence analysis turn invalid in the discrete scenario.

To bridge the graph from~\citet{golovin2019gradientless} to TN-SS, we first re-define several important concepts, by which the necessary tools for the analysis are re-derived.
In doing so, we begin by introducing the finite gradient as the alternative to the classic one defined in the continuous domain.

\begin{definition}[\textbf{finite gradient}]\label{apd:def:finiteGradient}
        For any function $f:\mathbb{Z}_{\geq{}0}^L\rightarrow{}\mathbb{R}$, its finite gradient $\Delta{}f:\mathbb{Z}_{\geq{}0}^L\rightarrow{}\mathbb{R}^{L}$ at the point $\mathbf{x}$ is defined as the vector
    \begin{equation}
    \begin{split}
        \Delta{}f(\mathbf{x})=
        \left[f(\mathbf{x}+\mathbf{e}_1)-f(\mathbf{x}),\ldots,f(\mathbf{x}+\mathbf{e}_L)-f(\mathbf{x})\right]^\top,\label{apd:eq:finiteGradient}
    \end{split}
    \end{equation}
    where $\mathbf{e}_i\,\forall{}i\in{}[L]$ denote the unit vectors with the $i$-th entry being one and other entries being zeros.
\end{definition}
Applying the finite gradient defined in~\eqref{apd:eq:finiteGradient}, we also re-define the strong convexity and smoothness for analysis in the discrete domain.

\begin{definition}[\textbf{$\alpha$-strong convexity with finite gradient}]\label{apd:def:strongConvexity}
        We say $f$ is $\alpha$-strongly convex for $\alpha\geq{}0$ if $f(\mathbf{y})\geq{}f(\mathbf{x})+\left<\Delta{}f(\mathbf{x})-{\color{blue}\frac{\alpha}{2}\mathbf{1}},\mathbf{y-x}\right>+\frac{\alpha}{2}\Vert{}\mathbf{y-x}\Vert^2$ for all $\mathbf{x,y}\in\mathbb{Z}_{\geq{}0}^L$, where $\mathbf{1}\in\mathbb{R}^L$ denotes the vector with all entries being one.
    We simply say $f$ is convex if it is $\alpha$-strongly convex and $\alpha=0$.
    \end{definition}
    Compared to the definitions used in~\citet{golovin2019gradientless} and other literature for convex analysis, the additional term, $\frac{\alpha}{2}\mathbf{1}$, marked by the {\color{blue}\emph{blue}} color, appears due to the discrepancy of the finite gradient and its counterpart in the continuous domain. 
    Below, we prove several basic results using the $\alpha$-strong convexity with finite gradient.
    
    
\begin{lemma}\label{apd:thm:convexity}
    If $f$ is $\alpha$-strongly convex in $\mathbb{Z}_{\geq{}0}^L$, then the following inequalities are held:
    \begin{enumerate}
        \item $g(\mathbf{x})=f(\mathbf{x})-\frac{\alpha}{2}\Vert{}\mathbf{x}\Vert^2$ is convex in the discrete scenario for all $\mathbf{x}\in{}\mathbb{Z}_{\geq{}0}^L$, and vice versa;
        \item $\left<\Delta{}f(\mathbf{x})-\Delta{}f(\mathbf{x}),\mathbf{x-y}\right>\geq{}\alpha\Vert{}\mathbf{x-y}\Vert^2$ for any $\mathbf{x,y}\in\mathbb{Z}_{\geq{}0}^L$;
        \item $\Vert\Delta{}f(\mathbf{x})-\Delta{}f(\mathbf{y})\Vert\geq{}\alpha\Vert{}\mathbf{x-y}\Vert$ for any $\mathbf{x,y}\in\mathbb{Z}_{\geq{}0}^L$;
    \end{enumerate}
    Here $\Vert\,\cdot\,\Vert$ denotes the $l_2$ norm for vectors.
\end{lemma}
\begin{proof}
    (1, $\Rightarrow$)~According to Def.~\ref{apd:def:strongConvexity}, the first statement is equivalent to proving the inequality
    \begin{equation}
    g(\mathbf{y})\geq{}g(\mathbf{x})+\left<\Delta{}g(\mathbf{x}),\mathbf{y-x}\right>\label{apd:eq:convexityG}
    \end{equation}
    holding for any $\mathbf{x,y}\in\mathbb{Z}_{\geq{}0}^L$.
    Applying the $\alpha$-strong convexity assumption, it follows that
    \begin{equation}
        \begin{split}
            g(\mathbf{y})-g(\mathbf{x})-\left<\Delta{}g(\mathbf{x}),\mathbf{y-x}\right>&=f(\mathbf{y})-\frac{\alpha}{2}\Vert{}\mathbf{y}\Vert^2-f(\mathbf{x})+\frac{\alpha}{2}\Vert{}\mathbf{x}\Vert^2-\left<\Delta{}g(\mathbf{x}),\mathbf{y-x}\right>\\
            &=f(\mathbf{y})-\frac{\alpha}{2}\Vert{}\mathbf{y}\Vert^2-f(\mathbf{x})+\frac{\alpha}{2}\Vert{}\mathbf{x}\Vert^2-\left<\Delta{}f(\mathbf{x})-\frac{\alpha}{2}(2\mathbf{x}+\mathbf{1}),\mathbf{y-x}\right>\\
            &={}f(\mathbf{y})-f(\mathbf{x})-\left<\Delta{}f(\mathbf{x})-\frac{\alpha}{2}\mathbf{1},\mathbf{y-x}\right>-\frac{\alpha}{2}\Vert{}\mathbf{y}\Vert^2+\frac{\alpha}{2}\Vert{}\mathbf{x}\Vert^2+\frac{\alpha}{2}\left<2\mathbf{x},\mathbf{y-x}\right>\\
            &\geq{}\frac{\alpha}{2}\left(\Vert{}\mathbf{y-x}\Vert^2-\Vert{}\mathbf{y}\Vert^2-\Vert{}\mathbf{x}\Vert^2+2\left<\mathbf{x,y}\right>\right)=0.
        \end{split}
    \end{equation}
    Here the first equality follows from the definition of $g(\mathbf{x})$, 
    the second equality holds since the finite gradient $\Delta\Vert{}\mathbf{x}\Vert^2=2\mathbf{x}+\mathbf{1}$, and
    the inequality at the bottom line follows from the $\alpha$-strong convexity assumption on $f$.
    
    (1, $\Leftarrow$) By~\eqref{apd:eq:convexityG},
    \begin{equation}
        \begin{split}
            f(\mathbf{y})-\frac{\alpha}{2}\Vert\mathbf{y}\Vert^2\geq{}f(\mathbf{x})-\frac{\alpha}{2}\Vert\mathbf{x}\Vert^2+\left<\Delta{}f(\mathbf{x})-\frac{\alpha}{2}(2\mathbf{x}+1),\mathbf{y-x}\right>.\label{apd:eq:convexity2G}
        \end{split}
    \end{equation}
    The $\alpha$-strong convexity of $f$ is thus proved by algebraically simplifying~\eqref{apd:eq:convexity2G}.

    (2)~To prove the second statement, we first know that the following inequality
    \begin{equation}
        \left<\Delta{}g(\mathbf{x})-\Delta{}g(\mathbf{y}),\mathbf{x-y}\right>\geq{}0,\,\forall{}\mathbf{x,y}
    \end{equation}
    holds since the monotone gradient property of the convexity (it is true in both continuous and discrete scenarios).
    By the form of $\Delta{}g(\mathbf{x})$, it follows that
    \begin{equation}
        \left<\Delta{}f(\mathbf{x})-\frac{\alpha}{2}(2\mathbf{x+1})-\Delta{}f(\mathbf{y})+\frac{\alpha}{2}(2\mathbf{y+1}),\mathbf{x-y}\right>\geq{}0.
    \end{equation}
    With algebraic simplification, we obtain
    \begin{equation}
        \left<\Delta{}f(\mathbf{x})-\Delta{}f(\mathbf{y}),\mathbf{x-y}\right>\geq{}\alpha\Vert{}\mathbf{x-y}\Vert^2\label{apd:eq:lemma2Convexity}
    \end{equation}
    for all $\mathbf{x,y}\in\mathbb{Z}_{\geq{}0}^L$. The second statement is thus proved.

    (3)~The third statement holds since the following relationship
    \begin{equation}
        \Vert\Delta{}f(\mathbf{x})-\Delta{}f(\mathbf{y})\Vert\Vert{}\mathbf{x-y}\Vert\geq{} \left<\Delta{}f(\mathbf{x})-\Delta{}f(\mathbf{y}),\mathbf{x-y}\right>\geq{}\alpha\Vert{}\mathbf{x-y}\Vert^2,
    \end{equation}
    where the first inequality follows from the Cauchy–Schwarz inequality, and the second inequality follows from~\eqref{apd:eq:lemma2Convexity}.
    The third statement is proved by dividing by $\Vert{}\mathbf{x-y}\Vert$ on both sides.
\end{proof}

Apart from the convexity, the smoothness of the objective function is also required to be re-defined in the discrete scenario.

\begin{definition}[\textbf{$(\beta_1,\beta_2)$-smoothness with finite gradient}]\label{apd:def:smoothness}
    We say $f$ is $(\beta_1,\beta_2)$-smooth for $\beta_1,\beta_2>0$ if
    \begin{enumerate}
        \item $\vert{}f(\mathbf{x})-f(\mathbf{y})\vert\leq{}\beta_1\Vert\mathbf{x-y}\Vert$ for all $\mathbf{x,y}\in\mathbb{Z}_{\geq{}0}^L$;
        \item The function $l(\mathbf{x}):=\frac{\beta_2}{2}\Vert{}\mathbf{x}\Vert^2-f(\mathbf{x})$ is convex.  
    \end{enumerate}
\end{definition}
The first item of Def.~\ref{apd:def:smoothness} restricts that $f$ is $\beta_1$-Lipschitz, implying the ``continuity'' of the function, while the second item upper bounds the change of the finite gradient of $f$, implying a ``continuity'' over the finite gradient.
In particular,
\begin{lemma}\label{apd:thm:smooth2Item}
    If $l(\mathbf{x})=\frac{\beta}{2}\Vert{}\mathbf{x}\Vert^2-f(\mathbf{x})$ is convex, then for all $\mathbf{x,y}\in\mathbb{Z}_{\geq{}0}^L$
    \begin{enumerate}
        \item $f(\mathbf{y})\leq{}f(\mathbf{x})+\left<\Delta{}f(\mathbf{x})-\frac{\beta}{2}\mathbf{1},\mathbf{y-x}\right>+\frac{\beta}{2}\Vert{}\mathbf{y-x}\Vert^2$ and vise versa;
        \item $\left<\Delta{}f(\mathbf{x})-\Delta{}f(\mathbf{y}),\mathbf{x-y}\right>\leq{}\beta\Vert{}\mathbf{x-y}\Vert^2$.
    \end{enumerate}
\end{lemma}
\begin{proof}
    (1, $\Rightarrow$)~By the form $l(\mathbf{x})$ and its convex property, we have the inequality
\begin{equation}
    \frac{\beta}{2}\Vert\mathbf{y}\Vert^2-f(\mathbf{y})\geq{}\frac{\beta}{2}\Vert\mathbf{x}\Vert^2-f(\mathbf{x})+\left<\frac{\beta}{2}(2\mathbf{x}+\mathbf{1})-\Delta{}f(\mathbf{x}),\mathbf{y-x}\right>.\label{apd:eq:smooth1Lemma}
\end{equation}
The first statement is proved by algebraically simplifying~\eqref{apd:eq:smooth1Lemma}.
The ($\Leftarrow$) direction can be proved similarly.

To prove the second item, by the convexity of $l(\mathbf{x})$,
    \begin{equation}
    l(\mathbf{y})\geq{}l(\mathbf{x})+\left<\Delta{}l(\mathbf{x}),\mathbf{y-x}\right>.\label{apd:eq:smooth21Lemma}
    \end{equation}
    Similarly,
    \begin{equation}
        l(\mathbf{x})\geq{}l(\mathbf{y})+\left<\Delta{}l(\mathbf{y}),\mathbf{x-y}\right>.\label{apd:eq:smooth22Lemma}
    \end{equation}
    Summing the two sides of~\eqref{apd:eq:smooth21Lemma} and~\eqref{apd:eq:smooth22Lemma} up, we have
    \begin{equation}
        \left<\Delta{}l(\mathbf{x})-\Delta{}l(\mathbf{y}),\mathbf{x}-\mathbf{y}\right>\geq{}0.
    \end{equation}
    Applying $l(\mathbf{x})=\frac{\beta}{2}\Vert{}\mathbf{x}\Vert^2-f(\mathbf{x})$,
    \begin{equation}
        \left<\beta(2\mathbf{x}+\mathbf{1})-\Delta{}f(\mathbf{x})-\beta(2\mathbf{y}+\mathbf{1})+\Delta{}f(\mathbf{y}),\mathbf{x-y}\right>\geq{}0.
    \end{equation}
    By simplifying the inequality, we finally have
    \begin{equation}
        \left<\Delta{}f(\mathbf{x})-\Delta{}f(\mathbf{y}),\mathbf{x-y}\right>\leq{}\beta\Vert{}\mathbf{x-y}\Vert^2.
    \end{equation}
\end{proof}

The first item of Def.~\ref{apd:def:smoothness} gives the following crucial result, which is used in the main theorem of this paper.
\begin{lemma}\label{apd:thm:smooth1Item}
    If $\vert{}f(\mathbf{x})-f(\mathbf{y})\vert\leq{}\beta\Vert\mathbf{x-y}\Vert$ for all $\mathbf{x,y}\in\mathbb{Z}_{\geq{}0}^L$, then the norm of the finite gradient with respective to $\mathbf{x}$ is bounded, \textit{i.e.}, $\Vert\Delta{}f(\mathbf{x})\Vert_\infty\leq{}\beta$.
\end{lemma}
\begin{proof}
    Denote $\Delta{}f(\mathbf{x})_i$ the $i$-th entry of $\Delta{}f(\mathbf{x})$, then for all $1\leq{}i\leq{}L$ the second item of the definition follows by
    \begin{equation}
        \vert\Delta{}f(\mathbf{x})_i\vert=\vert{}f(\mathbf{x}+\mathbf{e}_i)-f(\mathbf{x})\vert\leq{}\beta\Vert\mathbf{x}+\mathbf{e}_i-\mathbf{x}\Vert=\beta,
    \end{equation}
    where $\mathbf{e}_i$ denotes the unit vector with $i$-th entry being one and others being zeros, and the first equality follows from the definition of the finite gradient.
\end{proof}

After the new definitions of convexity and smoothness with finite gradient, in the proof, we also use the concept of the sub-level set, which is widely used in optimization theory.
For the self-consistency purpose, the specific definition is reviewed as follows:
\begin{definition}[\textbf{sub-level set}]
The level set of $f$ at point $\mathbf{x}\in\mathbb{Z}_{\geq{}0}^L$ is $\mathbb{L}_\mathbf{x}(f)=\left\{\mathbf{y}\in\mathbb{Z}_{\geq{}0}^L:\,f(\mathbf{y})=f(\mathbf{x})\right\}$.
The sub-level set of $f$ at point $\mathbf{x}\in\mathbb{Z}_{\geq{}0}^L$ is $\mathbb{L}^{\downarrow}_\mathbf{x}(f)=\left\{\mathbf{y}\in\mathbb{Z}_{\geq{}0}^L:\,f(\mathbf{y})\leq{}f(\mathbf{x})\right\}$.
\end{definition}

The following lemma shows that, for any $\mathbf{x}$, there exists a cube, \textit{i.e.}, a ball with infinity-norm, which is tangent at $\mathbf{x}$ and inside the sub-level set $\mathbb{L}^\downarrow_\mathbf{x}(f)$.

\begin{lemma}[\textbf{the sub-level cube}]\label{apd:thm:subLevelCube}
    Assume that  $f:\mathbb{Z}_{\geq{}0}^L\rightarrow{}\mathbb{R}$ is $\alpha$-strongly convex, $(\beta_1,\beta_2)$-smooth, and its minimum, denoted $f(\mathbf{\mathbf{x}^*})$, satisfies $\Vert\frac{\beta_2}{2}\mathbf{1}-\Delta{}f(\mathbf{x}^*)\Vert\leq{}\gamma$ where $\gamma$ is a constant and $0\leq{}\gamma<{}\alpha$.
    Then, for all $\mathbf{x}\in\mathbb{Z}_{\geq{}0}^L$, there is a $L$-dimensional cube, which is of the edge length $\frac{2(\alpha-\gamma)}{\beta_2\sqrt{L}}\Vert{}\mathbf{x-x}^*\Vert$, tangent at $\mathbf{x}$, and inside the sub-level set $\mathbb{L}^\downarrow_\mathbf{x}(f)$.
\end{lemma}
\begin{proof}
    Applying the smoothness assumption and Lemma~\ref{apd:thm:smooth2Item},
    \begin{equation}
    \begin{split}
        f\left(\mathbf{x}-\frac{1}{\beta_2}\Delta{}f(\mathbf{x})+\frac{1}{2}\mathbf{1}+\mathbf{s}\right)&\leq{}f(\mathbf{x})+\left<\Delta{}f(\mathbf{x})-\frac{\beta_2}{2}\mathbf{1},\mathbf{s}+\frac{1}{2}\mathbf{1}-\frac{1}{\beta_2}\Delta{}f(\mathbf{x})\right>+\frac{\beta_2}{2}\left\Vert{}\mathbf{s}+\frac{1}{2}\mathbf{1}-\frac{1}{\beta_2}\Delta{}f(\mathbf{x})\right\Vert^2\\
        &=f(\mathbf{x})+\frac{\beta_2}{2}\left(\Vert{}\mathbf{s}\Vert^2-\Vert{}\frac{1}{2}\mathbf{1}-\frac{1}{\beta_2}\Delta{}f(\mathbf{x})\Vert^2\right)
    \end{split}\label{apd:eq:sub1levelBall}
\end{equation}
for any $\mathbf{s}$.
The inequality~\eqref{apd:eq:sub1levelBall} implies that for any $\mathbf{y}\in\mathbb{Z}_{\geq{}0}^L$ in the Euclidean ball $B\left(\mathbf{x}-\frac{1}{\beta_2}\Delta{}f(\mathbf{x})+\frac{1}{2}\mathbf{1},\Vert{}\frac{1}{2}\mathbf{1}-\frac{1}{\beta_2}\Delta{}f(\mathbf{x})\Vert\right)$ it yields $f(\mathbf{y})\leq{}f(\mathbf{x})$, \textit{i.e.}, $\mathbf{y}\in\mathbb{L}^\downarrow_\mathbf{x}(f)$.
We also see that $\mathbf{x}$ is at the surface of this Euclidean ball, \textit{i.e.}, the ball is tangent at $\mathbf{x}$.
Furthermore, we also prove that the radius of the ball is lower bounded as follows:
\begin{equation}
    \begin{split}
        \frac{1}{\beta_2}\Vert{}\frac{\beta_2}{2}\mathbf{1}-\Delta{}f(\mathbf{x})\Vert
        &=\frac{1}{\beta_2}\Vert{}\frac{\beta_2}{2}\mathbf{1}-\Delta{}f(\mathbf{x}^*)+\Delta{}f(\mathbf{x}^*)-\Delta{}f(\mathbf{x})\Vert\\
        &\geq\frac{1}{\beta_2}\left(\Vert\Delta{}f(\mathbf{x})-\Delta{}f(\mathbf{x}^*)\Vert-\Vert\frac{\beta}{2}\mathbf{1}-\Delta{}f(\mathbf{x}^*)\Vert\right)\\
        &\geq{}\frac{(\alpha-\gamma)}{\beta_2}\Vert{}\mathbf{x-x}^*\Vert,
    \end{split}\label{apd:eq:sub2levelBall}
\end{equation}
where the inequality at the bottom line follows from the third statement of Lemma~\ref{apd:thm:convexity} and the assumption $\Vert\frac{\beta_2}{2}\mathbf{1}-\Delta{}f(\mathbf{x}^*)\Vert\leq{}\gamma$.

Next, we show that the ball $B\left(\mathbf{x}-\frac{1}{\beta_2}\Delta{}f(\mathbf{x})+\frac{1}{2}\mathbf{1},\Vert{}\frac{1}{2}\mathbf{1}-\frac{1}{\beta_2}\Delta{}f(\mathbf{x})\Vert\right)$ contains a cube of edge length $\frac{2(\alpha-\gamma)}{\beta_2\sqrt{L}}\Vert{}\mathbf{x-x}^*\Vert$.
First, we easily know in the ball there exists a cube, of which the volume is sufficiently small, and one vertex is at $\mathbf{x}$.
Then, the cube gradually extends all edges until the adjacent vertices of $\mathbf{x}$ touch the surface of the ball.
At this moment, it can be seen that the edges that touch the surface of the ball turn the chords of the ball.
Furthermore, the line connecting the ball center to $\mathbf{x}$ has an equal angle to all edges connecting $\mathbf{x}$, due to the symmetry of the geometrical shapes.
With basic geometry knowledge, we can thus calculate the chord length, \textit{i.e.}, the edge length of the cube, with $2\times{}R\cos(\theta)$, where $R$ denotes the radius of the ball and $\theta=\arccos(1/\sqrt{L})$ is the angle between the chord and the "center-$\mathbf{x}$" line.
Finally, using~\eqref{apd:eq:sub2levelBall}, we know the cube of the edge length $\frac{2(\alpha-\gamma)}{\beta_2\sqrt{L}}\Vert{}\mathbf{x-x}^*\Vert$ is tangent at $\mathbf{x}$, and inside sub-level set $\mathbb{L}^\downarrow_\mathbf{x}(f)$.
\end{proof}

\begin{lemma}[\textbf{convex combination in the discrete domain}]\label{apd:thm:convexComb}
    Suppose $\mathbf{q}=\theta{}\mathbf{x}+(1-\theta)\mathbf{y},\,\forall{}\theta\in{}[0,1]$, and there is $\hat{\mathbf{q}}\in{}\mathbb{Z}_{\geq{}0}^L$ where $\mathbf{\Lambda}=\mathbf{q}-\hat{\mathbf{q}}$.
    If $f$ is $\alpha$-strongly convex, then
    \begin{equation}
        \theta{}f(\mathbf{x})+(1-\theta)f(\mathbf{y})\geq{}f(\hat{\mathbf{q}})+\left<\Delta{}f(\hat{\mathbf{q}})-\frac{\alpha}{2}\mathbf{1},\mathbf{\Lambda}\right>+\frac{\alpha}{2}\Vert\mathbf{\Lambda}\Vert^2.
    \end{equation}
\end{lemma}
\begin{proof}
    By the definition of the $\alpha$-strong convexity,
    \begin{equation}
        \begin{split}
            &f(\mathbf{x})\geq{}f(\hat{\mathbf{q}})+\left<\Delta{}f(\hat{\mathbf{q}})-\frac{\alpha}{2}\mathbf{1},\mathbf{x}-\hat{\mathbf{q}}\right>+\frac{\alpha}{2}\Vert{}\mathbf{x}-\hat{\mathbf{q}}\Vert^2;\\
            &f(\mathbf{y})\geq{}f(\hat{\mathbf{q}})+\left<\Delta{}f(\hat{\mathbf{q}})-\frac{\alpha}{2}\mathbf{1},\mathbf{y}-\hat{\mathbf{q}}\right>+\frac{\alpha}{2}\Vert{}\mathbf{y}-\hat{\mathbf{q}}\Vert^2.
        \end{split}
    \end{equation}
    Thus, we have their convex combination as
    \begin{equation}
        \begin{split}
            \theta{}f(\mathbf{x})+(1-\theta)f(\mathbf{y})&\geq{}f(\hat{\mathbf{q}})+\left<\Delta{}f(\hat{\mathbf{q}})-\frac{\alpha}{2}\mathbf{1},\mathbf{\Lambda}\right>
            +\frac{\alpha}{2}\left(\theta\Vert{}\mathbf{x}\Vert^2+(1-\theta)\Vert{}\mathbf{y}\Vert^2+\Vert\hat{\mathbf{q}}\Vert^2-2\left<\mathbf{q},\hat{\mathbf{q}}\right>\right)\\
            &\geq{}f(\hat{\mathbf{q}})+\left<\Delta{}f(\hat{\mathbf{q}})-\frac{\alpha}{2}\mathbf{1},\mathbf{\Lambda}\right>
            +\frac{\alpha}{2}\left(\Vert{}\mathbf{q}\Vert^2+\Vert\hat{\mathbf{q}}\Vert^2-2\left<\mathbf{q},\hat{\mathbf{q}}\right>\right)\\
            &=f(\hat{\mathbf{q}})+\left<\Delta{}f(\hat{\mathbf{q}})-\frac{\alpha}{2}\mathbf{1},\mathbf{\Lambda}\right>
            +\frac{\alpha}{2}\Vert\mathbf{\Lambda}\Vert^2\\
        \end{split},
    \end{equation}
    where the second inequality follows from the convexity of $\Vert\,\cdot\,\Vert^2$.
    The proof is completed.
\end{proof}

\begin{assumption}\label{apd:def:assumption}
Assume that  $f:\mathbb{Z}_{\geq{}0}^L\rightarrow{}\mathbb{R}_+$ of~\eqref{apd:eq:generalModel} is $\alpha$-strongly convex, $(\beta_1,\beta_2)$-smooth, and its minimum, denoted $(p^*,\mathbf{x}^*)=\arg\min_{p,\mathbf{x}}f\circ{}p(\mathbf{x})$, satisfies $\Vert\Delta{}f_{p^*}(\mathbf{x}^*)-\frac{\beta_2}{2}\mathbf{1}\Vert\leq{}\gamma$ where $0\leq\gamma<\alpha\leq{}\beta_1\leq\beta_2\leq{}1$.
\end{assumption}

Here the inequality $\Vert\Delta{}f_{p^*}(\mathbf{x}^*)-\frac{\beta_2}{2}\mathbf{1}\Vert\leq{}\gamma$ implies that, up to a (small) constant vector $\frac{\beta_2}{2}\mathbf{1}$, the finite gradient at $(p^*,\mathbf{x}^*)$ should be sufficiently small, which can be understood as the discrete version of the zero-gradient for the stationary points in the continuous domain.
The upper bound ``$1$'' is arbitrarily chosen just for simplifying the calculation.
Also note that $\beta_2$ must be larger than $\alpha$ due to the fact $\alpha\Vert\mathbf{x-y}\Vert^2\leq{}\left<\Delta{}f(\mathbf{x})-\Delta{}f(\mathbf{y}),\mathbf{x-y}\right>\leq\beta_2\Vert\mathbf{x-y}\Vert^2$ (see Lemma~\ref{apd:thm:convexity} and Lemma~\ref{apd:thm:smooth2Item}).
With Assumption~\ref{apd:def:assumption}, we next prove that the local-sampling-based searching algorithm achieves the linear convergence rate up to a constant, \emph{if $p^*$ is known beforehand}.

\begin{theorem}[\textbf{convergence rate}]\label{apd:thm:rateFixedP}
    Suppose Assumption~\ref{apd:def:assumption} is satisfied, the operator $p$ in~\eqref{apd:eq:generalModel} is fixed to be $p^*$, and $0\leq{}\theta\leq{}1$.
    Then, for any $\mathbf{x}$ with $\Vert{}\mathbf{x-x}^*\Vert_\infty\leq{}c$, we can find a neighborhood $B_\infty(\mathbf{x},r_\mathbf{x})$ where $r_\mathbf{x}\geq\theta{}c+\frac{1}{2}$, such that there exist a element $\mathbf{y}\in{}B_\infty(\mathbf{x},r_\mathbf{x})$ satisfying
    \begin{equation}
        \begin{split}
            f_{p^*}(\mathbf{y})-f_{p^*}(\mathbf{x}^*)&\leq{}(1-\theta)(f_{p^*}(\mathbf{x})-f_{p^*}(\mathbf{x}^*))+\frac{7}{8}K.\label{apd:eq:rateFixedP}
        \end{split}
    \end{equation}
\end{theorem}
\begin{proof}
    First of all, since the operator $p$ is fixed to be $p^*$, the problem~\eqref{apd:eq:generalModel} can be equivalently simplified by removing the formulation of $p$ out of~\eqref{apd:eq:generalModel}, which is written as
    \begin{equation}
    \min_{\mathbf{x}\in\mathbb{Z}_+^K}f(\mathbf{x}),
\end{equation}
where $f:\mathbb{Z}_{\geq{}0}^{K}\rightarrow{}\mathbb{R}$ represents the objective function.\footnote{Here for brevity, we re-use the notation of $f$ without ambiguity since the main properties of $f$ are preserved up to the domain restricting from $\mathbb{Z}_{\geq{}0}^{L}$ to $\mathbb{Z}_{\geq{}0}^{K}$.}
By Lemma~\ref{apd:thm:convexComb}, we have the following inequality:
     \begin{equation}
        \begin{split}
            f(\hat{\mathbf{q}})-f(\mathbf{x}^*)\leq{}(1-\theta)(f(\mathbf{x})-f(\mathbf{x}^*))+\left<\frac{\alpha}{2}\mathbf{1}-\Delta{}f(\hat{\mathbf{q}}),\mathbf{\Lambda}\right>-\frac{\alpha}{2}\Vert{}\mathbf{\Lambda}\Vert^2.\label{apd:eq:rateFixedRateP1}
        \end{split}
    \end{equation}
    Next, we prove in the neighborhood $B(\mathbf{x},r_\mathbf{x})$ there exists an element $\mathbf{y}$, which belongs to as well the sub-level cube  tangent at $\hat{\mathbf{q}}$ knowing by Lemma~\ref{apd:thm:subLevelCube}, so that $f(\mathbf{y})\leq{}f(\hat{\mathbf{q}})$ holds.
    To do so, we first know that the distance between $\hat{\mathbf{q}}$ and $p_\mathbf{x}(\mathbf{x})$ satisfying
    \begin{equation}
        \begin{split}
            \Vert{}\mathbf{x}-\hat{\mathbf{q}}\Vert_\infty=\Vert{}\mathbf{x}-\mathbf{q}+\mathbf{\Lambda}\Vert_\infty\leq{}\Vert{}\mathbf{x-q}\Vert_\infty+\Vert{}\mathbf{\Lambda}\Vert_\infty
            =\theta\Vert{}\mathbf{x-x}^*\Vert_\infty+\Vert\mathbf{\Lambda}\Vert_\infty
            \leq{}\theta{}c+\frac{1}{2}
        \end{split}.
    \end{equation}
    Here the last inequality follows from  $\Vert\mathbf{\Lambda}\Vert_\infty\leq\frac{1}{2}$, which holds because $\hat{\mathbf{q}}\in\mathbb{Z}_{\geq{}0}^K$ can be always found by rounding the entries of $\mathbf{q}$ into the closest integers.
    We thus know from the inequality that the intersection between the sub-level cube tangent at $\hat{\mathbf{q}}$ and $B(\mathbf{x},r_\mathbf{x})$ is not empty if $r_\mathbf{x}\geq{}\theta{}c+\frac{1}{2}$, proving the existence of the $\mathbf{y}$.
    Last, we bound~\eqref{apd:eq:rateFixedRateP1} as follows:
    \begin{equation}
        \begin{split}
            f(\mathbf{y})-f(\mathbf{x}^*)
            &\leq{}f(\hat{\mathbf{q}})-f(\mathbf{x}^*)
            \leq{}(1-\theta)(f(\mathbf{x})-f(\mathbf{x}^*))+\left<\frac{\alpha}{2}\mathbf{1}-\Delta{}f(\hat{\mathbf{q}}),\mathbf{\Lambda}\right>-\frac{\alpha}{2}\Vert{}\mathbf{\Lambda}\Vert^2\\
            &\leq{}(1-\theta)(f(\mathbf{x})-f(\mathbf{x}^*))+\left\vert{}\left<\frac{\alpha}{2}\mathbf{1},\mathbf{\Lambda}\right>\right\vert+\left\vert{}\left<\Delta{}f(\hat{\mathbf{q}}),\mathbf{\Lambda}\right>\right\vert+\frac{\alpha}{2}\Vert\mathbf{\Lambda}\Vert^2\\
            &\leq{}(1-\theta)(f(\mathbf{x})-f(\mathbf{x}^*))+\frac{\alpha}{4}K
            +\Vert\Delta{}f(\hat{\mathbf{q}})\Vert_\infty\Vert\mathbf{\Lambda}\Vert_1
            +\frac{\alpha}{2}\Vert\mathbf{\Lambda}\Vert^2\\
            &\leq{}(1-\theta)(f(\mathbf{x})-f(\mathbf{x}^*))
            +\frac{\alpha}{4}K
            +\frac{\beta_1}{2}K
            +\frac{\alpha}{8}K\\
            &\leq{}(1-\theta)(f(\mathbf{x})-f(\mathbf{x}^*))+\frac{3\alpha+4\beta_1}{8}K\\
            &\leq{}(1-\theta)(f(\mathbf{x})-f(\mathbf{x}^*))+\frac{7}{8}K.
        \end{split}
    \end{equation}
    Here the inequality in the fourth line follows from Lemma~\ref{apd:thm:smooth1Item} and $\Vert\mathbf{\Lambda}\Vert_\infty\leq{}1/2$, and the inequality at the bottom line follows from Assumption~\ref{apd:def:assumption} that $\alpha<\beta_1\leq{}1$.
    The proof is thus completed.
\end{proof}

It is known from the proof that the constant $(7/8)K$ appearing in~\eqref{apd:eq:rateFixedP} is due to the fact $\Vert\mathbf{\Lambda}\Vert_1\leq{}K\Vert\mathbf{\Lambda}\Vert_\infty\leq{}K/2$ and $\Vert\mathbf{\Lambda}\Vert_2\leq{}K\Vert\mathbf{\Lambda}\Vert_\infty\leq{}\sqrt{K}/2$.
It means that with the rounding error $\Vert\mathbf{\Lambda}\Vert_\infty\leq{}1/2$, the $l_{1,2}$ norm of $\mathbf{\Lambda}$ would become larger with increasing the dimension $K$, which is inevitable in the analysis.
It only disappears if $\Vert\mathbf{\Lambda}\Vert_\infty=0$, implying the conventional convex optimization in the continuous domain.

As an important corollary from Theorem~\ref{apd:thm:rateFixedP}, we next prove the convergence guarantee for the local-sampling-based methods.

\begin{corollary}[\textbf{convergence guarantee}]\label{apd:thm:guarantee}
    Suppose $p^*$ is known and a series $\left\{\mathbf{x}_n\right\}_{n=0}^\infty$, where $\mathbf{x}_0$ is randomly chosen in $\mathbb{Z}_+^K$, and for each $n>0$, $\mathbf{x}_n$ is equal to the $\mathbf{y}$ in Theorem~\ref{apd:thm:rateFixedP}.
    Then we can achieve the following limit when
    $\Omega(1/K)\leq\theta\leq{}1$,
    \begin{equation}
        \lim_{n\rightarrow{}\infty}\left(f_{p^*}(\mathbf{x}_n)-f_{p^*}(\mathbf{x}^*)\right)=O(1)
    \end{equation}
\end{corollary}
\begin{proof}
    Let $C_K:=(7/8)K$. By the updating rule,
    \begin{equation}
        \begin{split}
            f_{p^*}(\mathbf{x}_n)-f_{p^*}(\mathbf{x}^*)&\leq{}(1-\theta)(f_{p^*}(\mathbf{x}_{n-1})-f_{p^*}(\mathbf{x}^*))+C_K\\
            &\leq{}(1-\theta)^2(f_{p^*}(\mathbf{x}_{n-2})-f_{p^*}(\mathbf{x}^*))+C_K+C_K(1-\theta)\\
            &\leq{}(1-\theta)^3(f_{p^*}(\mathbf{x}_{n-3})-f_{p^*}(\mathbf{x}^*))+C_K+C_K(1-\theta)+C_K(1-\theta)^2\\
            &\leq{}\cdots\\
            &\leq{}(1-\theta)^{n}(f_{p^*}(\mathbf{x}_0)-f_{p^*}(\mathbf{x}^*))+C_K\sum_{m=1}^n(1-\theta)^{m-1}.
        \end{split}
    \end{equation}
    Thus using the condition $\Omega(1/K)\leq{}\theta\leq{}1$, we finally obtain that 
    \begin{equation}
        \begin{split}
            \lim_{n\rightarrow{}\infty}\left(f_{p^*}(\mathbf{x}_n)-f_{p^*}(\mathbf{x}^*)\right)\leq{}0+C_K\frac{1}{\theta}=O(1).
        \end{split}
    \end{equation}
\end{proof}

\subsection{Sampling efficiency}

\begin{proposition}[\textbf{curse of dimensionality for TNLS}]\label{apd:thm:TNLSSammpling}
Let the assumptions in Theorem~\ref{apd:thm:rateFixedP} be satisfied.
Furthermore, assume that $\mathbf{x}^*$ is sufficiently smaller (or larger) than $\mathbf{x}$ entry-wisely except for a constant number of  entries.
Then the probability of achieving a suitable $\mathbf{y}$ as mentioned in Theorem~\ref{apd:thm:rateFixedP} by uniformly randomly sampling in $B_\infty(\mathbf{x},r_\mathbf{x})$ with $r_\mathbf{x}\geq{}\theta{}c+\frac{1}{2}$ equals $O(2^{-K})$.
\end{proposition}
\begin{proof}
    We only prove the case where $\mathbf{x}^*$ is sufficiently smaller than $\mathbf{x}$ in the entry-wise manner, except a constant number of entries.
    The ``larger'' case can be proved similarly.
    Recall Theorem~\ref{apd:thm:rateFixedP}.
    By the construction of $\mathbf{y}$, we have $\mathbf{q}=\theta\mathbf{x}+(1-\theta)\mathbf{x}^*$ with $0\leq{}\theta\leq{}1$ and the approximation $\hat{\mathbf{q}}\in\mathbb{Z}_+^K$ with $\hat{\mathbf{q}}=\mathbf{q}+\mathbf{\Lambda}$  and $\Vert\mathbf{\Lambda}\Vert_\infty\leq{}1/2$.
    According to the assumptions, we know $\mathbf{x}-\hat{\mathbf{q}}$ is entry-wisely larger than zero except $C$ entries, where $C\geq{}0$ is a constant.
    Since $r_\mathbf{x}\geq{}\theta{}c+\frac{1}{2}$, we further know from Theorem~\ref{apd:thm:rateFixedP} that the intersection between $B_\infty(x,r_\mathbf{x})$, denoted $B$ in the rest of the proof for brevity, and the sub-level cube, denoted $A$, tangent at $\hat{\mathbf{q}}$ is not empty.
    In this case, we can easily bound the volume of the cube associated to the intersection of $A$ and $B$ as follows:
    \begin{equation}
        \begin{split}
            \vert{}A\cap{}B\vert\leq{}\left(r_\mathbf{x}-\delta_{\min}\right)^{K-C}\left(r_\mathbf{x}+\delta_{\max}\right)^C
        \end{split}.
    \end{equation}
     Here $\vert\,\cdot\,\vert$ denotes the volume of the cube.
    $\delta_{\min}=\min\left\{p_i:\,p_i=\mathbf{x}(i)-\hat{\mathbf{q}}(i)>0,\,1\leq{}i\leq{}K\right\}$
    and $\delta_{\max}=\max\left\{0,p_i:\,p_i=\hat{\mathbf{q}}(i)-\mathbf{x}(i)\leq{}0,1\leq{}i\leq{}K\right\}$,
    where $\mathbf{x}(i),\hat{\mathbf{q}}(i)$ denote the $i$-th entry of $\mathbf{x}$ and $\hat{\mathbf{q}}$, respectively.
    Thus, the probability of uniformly drawing a sample $\mathbf{y}$ belonging to $A\cap{}B$ from $B_\infty(\mathbf{x},r_\mathbf{x})$ is as follows:
    \begin{equation}
        \begin{split}
            Pr(y\in{}A\cap{}B)&\leq\frac{\left(r_\mathbf{x}-\delta_{\min}\right)^{K-C}\left(r_\mathbf{x}+\delta_{\max}\right)^C}{(2r_\mathbf{x})^K}\leq{}\left(\frac{r_\mathbf{x}+\delta_{\max}}{r_\mathbf{x}-\delta_{\min}}\right)^C{}2^{-K}\\
            &=O(2^{-K}).
        \end{split}
    \end{equation}
    The proof is completed.
\end{proof}

Recall that let $\mathbb{B}:=B(p)\times{}B_\infty(\mathbf{x},r_\mathbf{x})$ and $f_\mathbb{B}^*:=\min_{(p_y,\mathbf{y})\in{}\mathbb{B}}f_{p_y}(\mathbf{y})$ for notational simplicity, then

\begin{proposition}[\textbf{evaluation efficiency for TnALE}]\label{apd:thm:TnALSSampling}
    Let $\mathcal{B}\in\mathbb{R}^{I\times{}I\times{}\cdots\times{}I}$ be the tensor of order-$(K+1)$ constructed as Eq.~\eqref{eq:B2B} with $I_1=I_2=\cdots=I_{K+1}=I$.
    Then, there exists its TT-cross approximation~\cite{oseledets2010tt} of rank-$R$\footnote{Here we assume that all elements of the TT-ranks are equal to $R$ for brevity.}, denoted $\hat{\mathcal{B}}$,
    for which it satisfies $\mathbf{j}=\arg\max_{\mathbf{i}}\hat{\mathcal{B}}(\mathbf{i})$, such that the equality $f_\mathbb{B}^*=f_{p_{j_{K+1}}}\left(\mathbf{x}+\mathbf{j}(:K)-(\lceil{}r_\mathbf{x}\rceil+1)\right)$ holds, provided that
    \begin{equation}
 f_\mathbb{B}^*\leq{}f_{p_z}(\mathbf{z})/\left(1+2\frac{(4R)^{\lceil{}\log_2{}K\rceil}-1}{4R-1}(R+1)^2\xi{}f_{p_z}(\mathbf{z})\right)\label{apd:eq:y*}
    \end{equation}
    for all $(p_z,\mathbf{z})\in{}\mathbb{B}$ and $f_{p_z}(\mathbf{z})\neq{}f_\mathbb{B}^*$.
    Here, $\xi$ denotes the error between $\mathcal{B}$ and its best approximation of TT-ranks $R$ in terms of $\Vert\,\cdot\,\Vert_\infty$.
    Note that the inequality~\eqref{apd:eq:y*} holds trivially if $\mathcal{B}$ is exactly of the TT topology of rank-$R$, and~\citet{oseledets2010tt} shows that the $f_\mathbb{B}^*$ can be recovered from $O(KIR)$ entries from $\mathcal{B}$.
\end{proposition}
\begin{proof}
    Since the ``one-to-one'' relation between the entries of the tensor $\mathcal{B}$ and all possible $f(\mathbf{z})$ for all $(p_z,\mathbf{z})\in{}B(p)\times{}B_\infty(\mathbf{x},r_\mathbf{x})$,
    it is easily to know the equality $f_\mathbb{B}^*=f_{p_{j_{,K+1}}}\left(\mathbf{x}+\mathbf{j}(:K)-(\lceil{}r_\mathbf{x}\rceil+1)\right)$ holds if $\hat{\mathcal{B}}(\mathbf{i}^*)\geq{}\hat{\mathcal{B}}(\mathbf{k})$ for $\mathbf{i}^*=\arg\max_{\mathbf{i}}\mathcal{B}(\mathbf{i})$ and any index $\mathbf{k}$.
     To prove this condition true, we have the following inequalities for any $\mathbf{k}$:
     \begin{equation}
       \begin{split}
           \hat{\mathcal{B}}(\mathbf{i}^*)-\hat{\mathcal{B}}(\mathbf{k})&\geq{}\mathcal{B}(\mathbf{i}^*)-\mathcal{B}(\mathbf{k})-2\frac{(4R)^{\lceil{}\log_2{}K\rceil}-1}{4R-1}(R+1)^2\xi\\
           &=1/f_\mathbb{B}^*-1/f_{p_{j_{K+1}}}(\mathbf{x}+\mathbf{k}(:K)-(\lceil{}r_\mathbf{x}\rceil+1))-2\frac{(4R)^{\lceil{}\log_2{}K\rceil}-1}{4R-1}(R+1)^2\xi\\
           &\geq{}2\frac{(4R)^{\lceil{}\log_2{}K\rceil}-1}{4R-1}(R+1)^2\xi-2\frac{(4R)^{\lceil{}\log_2{}K\rceil}-1}{4R-1}(R+1)^2\xi=0
       \end{split},
   \end{equation}
   where the first inequality follows from Theorem 2 in~\cite{osinsky2019tensor}, and the last inequality follows from the inequality~\eqref{apd:eq:y*}.
    It can also be known if  $\mathcal{B}$ is exactly of the TT topology of rank-$R$, $\hat{\mathcal{B}}$ is able to recover $\mathcal{B}$ exactly.
    In this case $\xi=0$ and $f_\mathbb{B}^*\leq{}f(\mathbf{z})$ trivially for all $\mathbf{z}$.
\end{proof}

\newpage
\section{Experiment details}\label{apd:sec:experiment}
\subsection{Low-rank structure of the optimization landscape }\label{apd:sec:lowRankLandsacpe}
To verify the low-rank structure of the optimization landscape of~\eqref{eq:basicModel}, we empirically check the singular values of the landscape tensor using the synthetic data. 
To be specific, we re-use the fourth-order tensor in the experiment for TN-PS, \textit{i.e.}, TR~(order-4) in Table~\ref{apd:tab:TNPS:TR}.
Here we remove the influence of unknown permutations and calculate the objective for all possible combinations of values of the TN-ranks.
As a result, for each data, we have a landscape tensor (a tensor whose entries are values of the objective function) of order-$4$, and the modes of the tensor corresponding to the four TN-ranks.
Figure~\ref{apd:fig:landscape} (a) shows the singular values of the landscape tensor unfolded along different modes on average.
We see that the landscape tensor provides a \emph{significant} low-rank structure in the data.
We also depict the complete landscape (contour line, unfolded along the first two modes) with respect to Data A in Figure~\ref{apd:fig:landscape} (b).
We can see that the obviously repeated pattern shown in the figure is the main reason leading to the low-rank structure of the landscape.

\begin{figure*}[htbp]
\centering
\subfigure[Averaged singular values for the 4th-order landscape tensor.]{
\includegraphics[width=9.35cm]{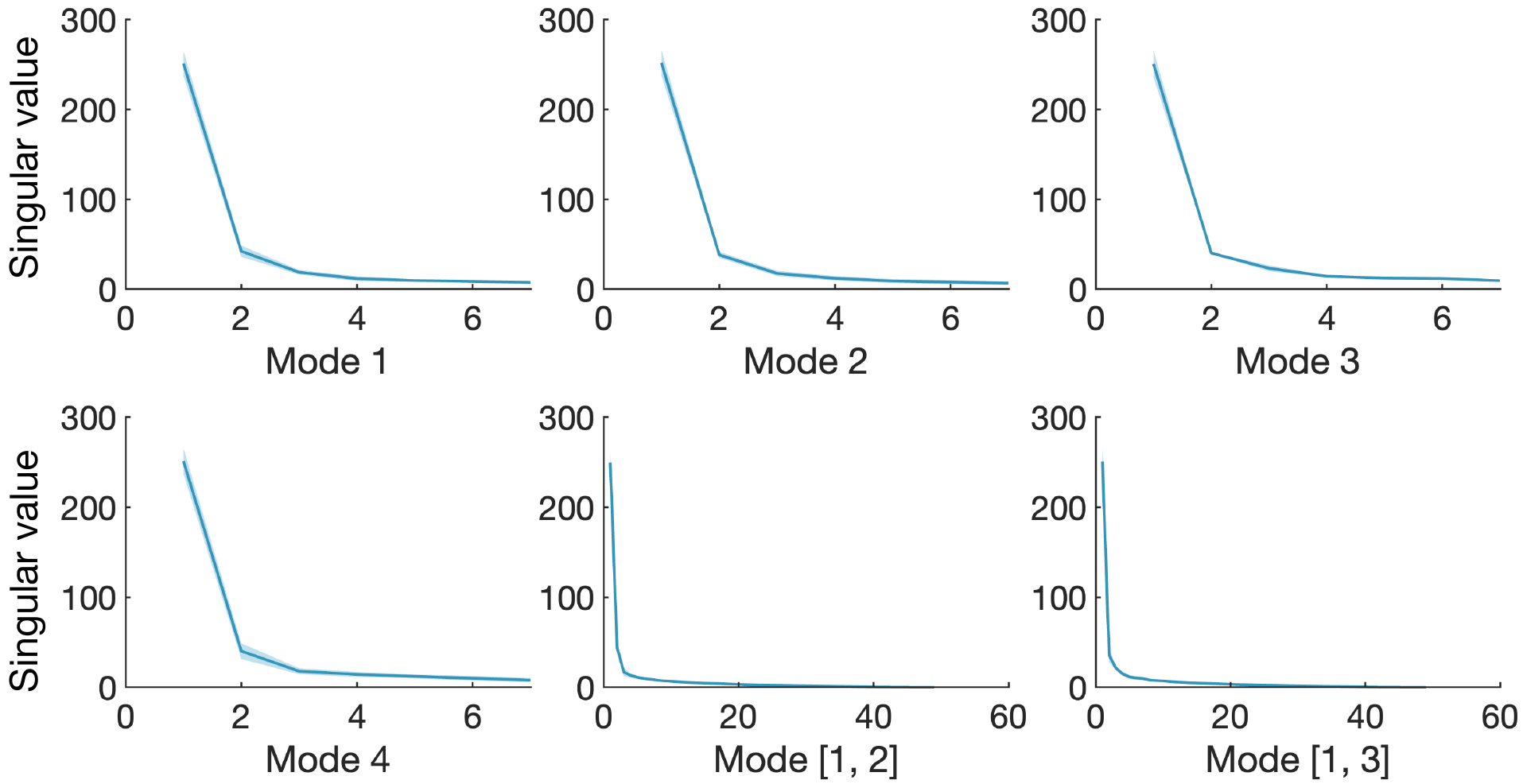}

}
\subfigure[Optimization landscapes (the inverse 1/f (x)) wrt. the tensor of order-4 and correct permutation.]{
\includegraphics[width=7.35cm]{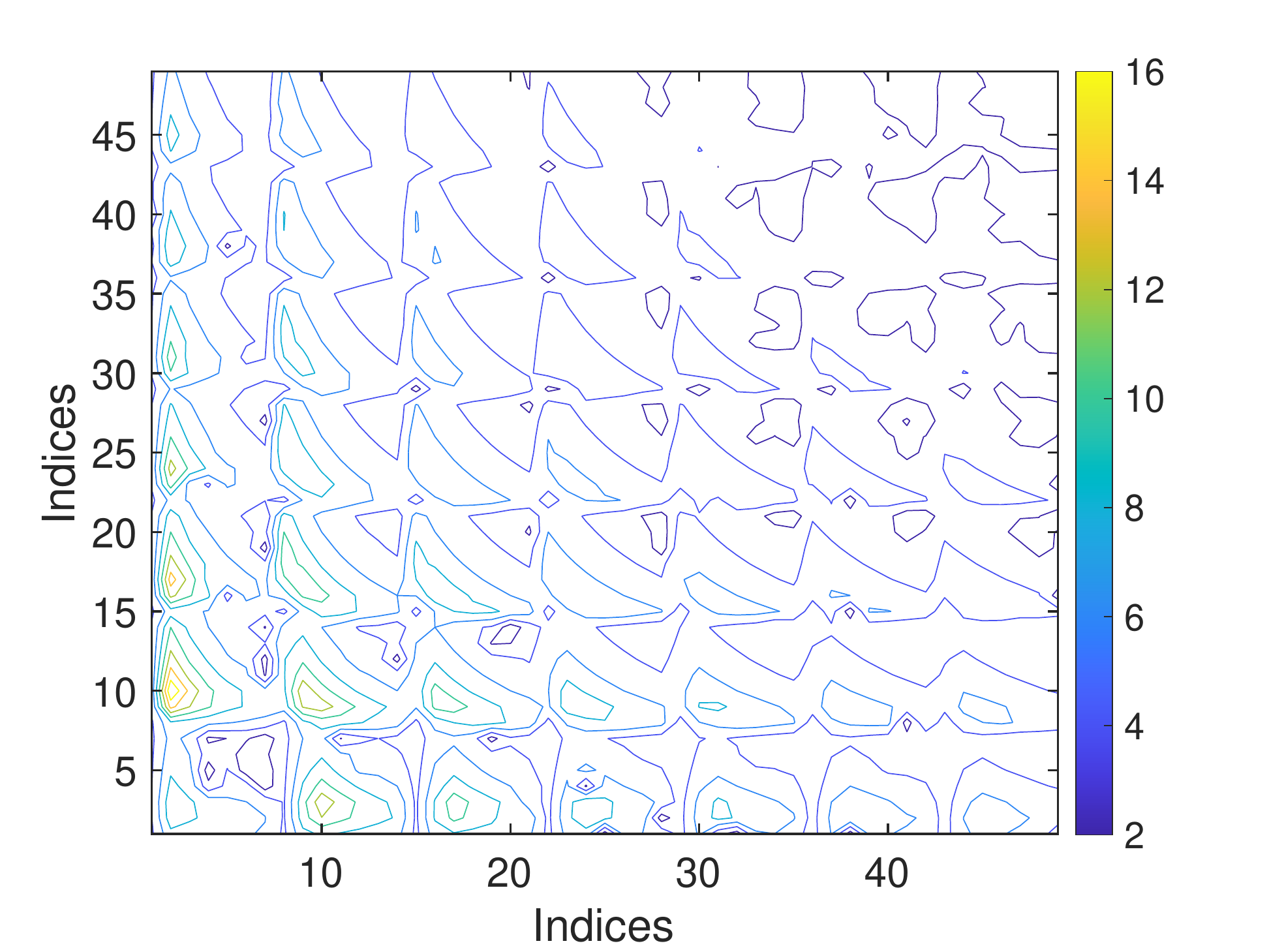}
}
\vspace{-0.4cm}
\caption{Averaged singular values and Optimization landscapes for the tensor of order-4.}
\label{apd:fig:landscape}
\end{figure*}

\subsection{Details for the experiment of TN-PS (\textit{w.r.t.}, Table~\ref{Tab:TNPS}).}
\textbf{Goal.}
In this experiment, our goal is to verify the superiority of TnALE in addressing the TN-PS problem.

\textbf{Data generation.}
For the synthetic data with TR topology (order-4, order-6, and order-8), as well as PEPS( order-6), HT (order-6), and MERA (order-8), we \emph{re-use} the data from ~\citet{li2022permutation}. To generate data with TW (order-5) topology, we set the dimension of each tensor mode to 3. Additionally, we randomly select TN-ranks from the set $\{1, 2, 3\}$. Then we \textit{i.i.d.} draw samples from Gaussian distribution $N(0, 1)$ as the values of core tensors. After contracting these core tensors based on the TW topology, we randomly and uniformly permute the tensor modes. 

\textbf{Settings.}
In the experiment, we implement TNGA and TNLS as comparison methods. We use the same objective function as described in~\citet{li2020evolutionary} for all the methods.
Specifically, the objective function of~\eqref{eq:basicModel} used in the experiment is as follows:
\begin{equation}
    F(G,\mathbf{r})=
    \underbrace{\frac{1}{\epsilon(G,\mathbf{r})}}_{\mbox{compression ratio (CR)}}+ \lambda\cdot\underbrace{\min_{\mathcal{Z}\in{}TNS(G,\mathbf{r})}\left\Vert\mathcal{X}-\mathcal{Z}\right\Vert^2/\left\Vert{}\mathcal{X}\right\Vert^2}_{\mbox{relative squared error (RSE)}},
    \label{eq:compressionlModel}
\end{equation}
where $\mathcal{X}$ denotes the synthetic tensor, and $\epsilon(G,\mathbf{r})$ represents the compression ratio equalling to
\begin{equation*}
    \epsilon(G,\mathbf{r})=\frac{\mbox{Dimension of }\mathcal{X}}{\mbox{Dimension sum of core tensors of the TN under }(G,\mathbf{r})}.
\end{equation*}

The trade-off parameter $\lambda$ in ~\eqref{eq:compressionlModel} is set to 200. For the solver of the inner minimization, we utilize the Adam optimizer ~\citet{kingma2014adam} with a learning rate of 0.001. Additionally, the core tensors are initialized using Gaussian distribution $N(0, 0.1)$. Furthermore, the search range for TN-ranks is set from 1 to 7, except for TW data, for which the search range is limited to 1 to 4. For TNGA, the maximum number of generations is set to 30. The population size in each generation is 120 for all the TN topologies except for TR, which is set as 150. During each generation, the elimination rate is 36$\%$ and the reproduction trick ~\cite{snyder2006random} is adopted and we set the reproduction number to be 2. Meanwhile, for the selection probability of the recombination operation, we set the hyper-parameters $\alpha=20$ and $\beta=1$. Moreover, there is a 24$\%$ chance for each gene to mutate after the recombination. For TNLS, we set the sample numbers in each local sampling stage to 60. The tuning parameter $c_{1}$ is fixed at 0.9 throughout the experiment. As for the tuning parameter $c_{2}$, it is adjusted based on the tensor order. Specifically, we set $c_{2}=0.9$ for order-4 TR, $c_{2}=0.94$ for order-6 TR, PEPS, TW and HT, and for MERA and order-8 TR, we set $c_{2}=0.98$. In our proposed method TnALE, we maintain consistent settings throughout the experiment. The rank-related radius is set as $r_{1}=2$ and $r_{2}=1$. During the initialization phase, we perform 2 iterations, and during the searching phase, we conduct 30 iterations. Additionally, we set the number of round-trips of ALE to 1. 
For performance evaluation, we use the \textit{Eff.} index, and \textit{Eff.}$\geq 1$ indicates an identical or more compact structure has been found. If the results do not satisfy the conditions of RSE $\le 10^{-4}$ and \textit{Eff.}$\geq 1$, we say the approach fails in the experiment. 
\begin{table*}[t]
	\centering
	\caption{Experimental results of the TN-PS task on TR topology. In the table, \emph{Eff.} and the required evaluation numbers \emph{\#Eva.} are demonstrated. Specifically, \emph{\#Eva.} is shown in the square brackets.
	}
	\begin{threeparttable}\tiny\label{apd:tab:TNPS:TR}
		\setlength{\tabcolsep}{0.5mm}{   	
			\begin{tabular}{cccccc|ccccc|ccccc}
				\toprule
	\multirow{4}[0]{*}{\textbf{Methods}}&\multicolumn{5}{c}{\textbf{ Order 4}}&\multicolumn{5}{c}{\textbf{order 6}}&\multicolumn{5}{c}{\textbf{order 8}}\\
				& \textbf{A} &  \textbf{B} & \textbf{C} & \textbf{D} & \textbf{E}&\textbf{A} &  \textbf{B} & \textbf{C} & \textbf{D} & \textbf{E}&\textbf{A} &  \textbf{B} & \textbf{C} & \textbf{D} & \textbf{E}\\
				\cmidrule{2-16}
				{}&\multicolumn{15}{c}{\emph{Eff.$\uparrow$~[$\#$Eva.$\downarrow$]} }\\
				\midrule
				\textbf{TNGA} &1.00~[450]& 1.00~[450] & 1.17~[450] & 1.00~[300] &1.00~[450] & 1.00~[1500] &1.00~[1350]& 1.00~[1650] & 1.16~[1650] & 1.00~[1050]&  1.00~[2850]&1.02~[2250] &1.11~[3950] & 1.06~[1950]&0.88~[1500] \\
				\textbf{TNLS} &1.00~[240]& 1.00~[300] & 1.17~[60] & 1.00~[300] &1.00~[360] & 1.00~[660] &1.00~[600]& 1.00~[660] & 1.16~[600] & 1.00~[540]&  1.00~[1020]&1.02~[960] &1.11~[1320] & 1.06~[780]&1.17~[900] \\
				\textbf{TnALE(ours)} &1.00~[\textbf{93}]& 1.00~[\textbf{155}] & 1.17~[\textbf{31}] & 1.00~[\textbf{124}] &1.00~[\textbf{62}] & 1.00~[\textbf{156}] &1.00~[\textbf{321}]& 1.00~[\textbf{156}] & 1.16~[\textbf{156}] & 1.00~[\textbf{89}]&  1.00~[\textbf{231}]&1.02~[\textbf{308}] &1.11~[\textbf{308}] & 1.06~[\textbf{231}]&1.17~[\textbf{178}]  \\
				

				\bottomrule
			\end{tabular}
		}
	\end{threeparttable}
\end{table*}

\begin{table*}[t]
	\centering
	\caption{Experimental results of the TN-PS task on PEPS, HT, MERA and TW topology. In the table, \emph{Eff.} and the required evaluation numbers \emph{\#Eva.} are demonstrated. Specifically, \emph{\#Eva.} is shown in the square brackets. The symbol “-” in the table means the failure of the approach.
	}
	\begin{threeparttable}\tiny\label{apd:tab:TNPS:others}
		\setlength{\tabcolsep}{0.25mm}{   	
			\begin{tabular}{ccccc|cccc|cccc|cccc}
				\toprule
	\multirow{4}[0]{*}{\textbf{Methods}}&\multicolumn{4}{c}{\textbf{ PEPS}}&\multicolumn{4}{c}{\textbf{HT}}&\multicolumn{4}{c}{\textbf{MERA}}&\multicolumn{4}{c}{\textbf{TW}}\\
				& \textbf{A} &  \textbf{B} & \textbf{C} & \textbf{D} &\textbf{A} &  \textbf{B} & \textbf{C} & \textbf{D} &\textbf{A} &  \textbf{B} & \textbf{C} & \textbf{D} &\textbf{A} &  \textbf{B} & \textbf{C} & \textbf{D} \\
				\cmidrule{2-17}
				{}&\multicolumn{16}{c}{\emph{Eff.$\uparrow$~[$\#$Eva.$\downarrow$]} }\\
				\midrule
				\textbf{TNGA} &1.14~[1560]& - & 1.00~[840] & 1.21~[1080] &1.45~[960] & 1.21~[1320] &1.18~[840]& 1.29~[1080] & - & 1.32~[960]&  2.30~[2800]&1.00~[3240] &1.24~[1920] & 2.61~[1440]&1.23~[600]&1.30~[720] \\
				\textbf{TNLS} &1.14~[781]& 1.00~[781] & 1.00~[421] & 1.21~[661] &1.45~[841] & 1.21~[841] &1.18~[781]& 1.29~[721] & 1.09~[1561] & 1.88~[841]&  2.88~[1441]&1.03~[721] &1.24~[661] & 2.61~[601]&1.23~[601]&1.30~[481] \\
				\textbf{TnALE(ours)} &1.14~[\textbf{407}]& 1.00~[\textbf{465}] & 1.00~[\textbf{233}] & 1.21~[\textbf{175}] &1.45~[\textbf{211}] & 1.21~[\textbf{281}] &1.18~[\textbf{211}]& 1.29~[\textbf{211}] & 1.09~[\textbf{1450}] & 1.88~[\textbf{484}]&  2.88~[\textbf{323}]&1.03~[\textbf{323}] &1.24~[\textbf{285}] & 2.61~[\textbf{143}]&1.23~[\textbf{285}]&1.30~[\textbf{214}]  \\
				

				\bottomrule
			\end{tabular}
		}
	\end{threeparttable}
\end{table*}

\begin{figure}[ht]
\centering
\includegraphics[width=1\columnwidth]{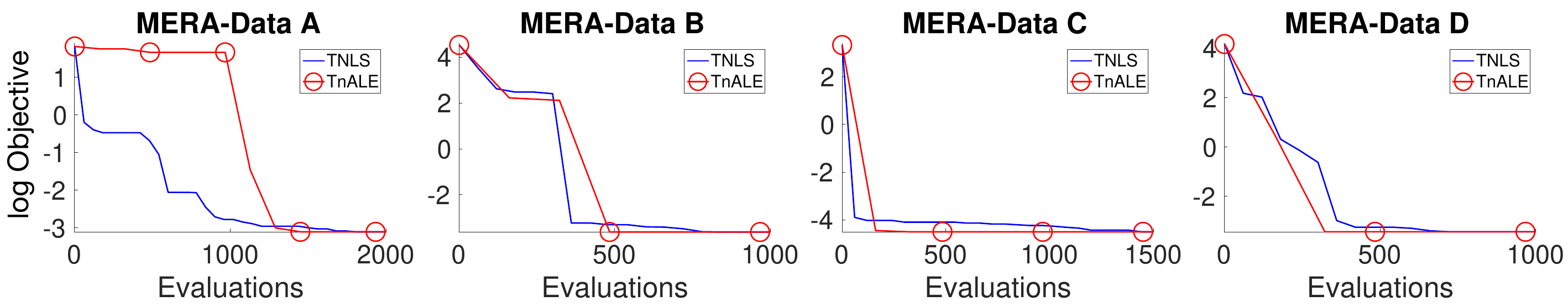}
\vspace{-0.4cm}
\caption{Objective (in the log form) with varying the number of evaluations: an observation of the local convergence of TnALE in MERA.}
\label{apd:fig:TNPS:Mera}
\end{figure}
\textbf{Results.}
The results for TR topology are presented in Table ~\ref{apd:tab:TNPS:TR}, and the results for PEPS, HT, MERA, and TW topology are shown in Table ~\ref{apd:tab:TNPS:others}. Based on the results, we observe that both TNLS and TnALE can successfully identify the ranks and permutations of the data, as indicated by \textit{Eff.}$\geq 1$. When comparing TNLS and TnALE, we find that TnALE achieves the same results with significantly fewer evaluation requirements. This highlights the superiority of TnALE in solving the TN-PS problem, demonstrating its efficiency and effectiveness. In Figure ~\ref{fig:TNPS:AverageCurves}, the averaged log objective curves with varying evaluation numbers of TNLS and TnALE are displayed. It is apparent from the figures that TnALE demonstrates a faster descending trend and achieves lower objective values given the same number of evaluations compared to TNLS for most cases (except for MERA). These results indicate the practical advantage of the proposed method, particularly in scenarios where computational resources are limited, and only a certain number of evaluations can be performed. For the results of MERA, we further draw the objective curves of each data in Figure 
~\ref{apd:fig:TNPS:Mera}. From the MERA-Data A curve, it is observed that TnALE descends at a slow pace until approximately 1000 evaluations, whereas TNLS continues to descend. The main reason for this behavior is that TnALE gets trapped in a local optimum and struggles to jump out by restarting the ALE algorithm with a new random center, while TNLS is more likely to overcome such local optima due to its stochastic essence. Moreover, in order to demonstrate the scalability of different TN-PS methods with respect to the tensor order, we draw the average number of evaluations with TR order in Figure ~\ref{apd:fig:TNPS:TR}. From the results, it is evident that the proposed method exhibits a slower increase in the number of evaluations with increasing tensor order compared to other methods. These results highlight the scalability of the proposed method, indicating its ability to handle higher-order tensors effectively.

\begin{figure}[ht]
    \centering
    \includegraphics[width=0.8\columnwidth]{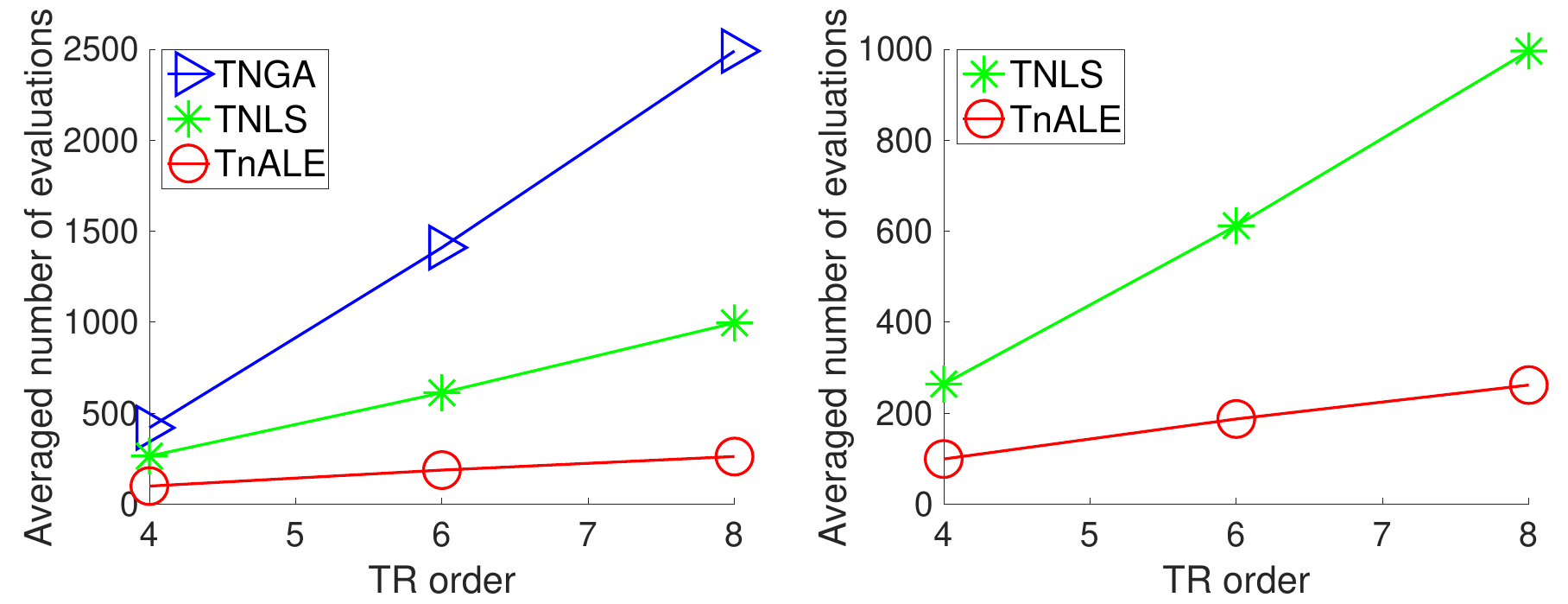}
    \vspace{-0.1cm}
    \caption{Number of evaluations with varying TR orders.}
    \label{apd:fig:TNPS:TR}
\end{figure}

\subsection{Details for the experiment of TN-RS (\textit{w.r.t.} Table~\ref{Tab:TNRS}).}
\textbf{Goal.}
In this experiment, we consider the classic rank-selection problem, \textit{i.e.}, TN-RS, for TR decomposition.

\textbf{Data Generation.}
We generate synthetic tensors in TR topology with two configurations:
``\emph{lower-ranks}'' and ``\emph{higher-ranks}''.
In both configurations, we generate five tensors by randomly selecting ranks and values of the vertices (core tensors). Each tensor has an order of 8, and the dimensions for each tensor mode are set to 3. We \textit{i.i.d.} draw samples from Gaussian distribution \textit{N} (0, 1) as the values of the vertices.
In the ``lower-ranks'' group, we uniformly select the TN-ranks from the interval $[1, 4]$ randomly, while in the ``higher-ranks'' group, we increase the rank interval to $[5,8]$. This ensures that the ranks would be larger than the dimensions of the tensor modes. This configuration aims to simulate the scenario of the over-determined ranks, which commonly occurs in practice for high-order TNs but has received limited attention in existing works.

\textbf{Settings.} 
In the experiment, we compare various rank-adaptive TR decomposition methods. These methods include TR-SVD, TR-rSVD, TR-ALSAR, TR-BALS and TR-BALS2~\cite{zhao2016tensor}, TR-LM (Alg. 2 and Alg. 3)~\cite{mickelin2020algorithms}, TRAR~\cite{sedighin2021adaptive}.
Additionally, the TTOpt algorithm~\cite{sozykin2022ttopt} with ranks~\footnote{Here the ranks are tuning parameters in the TTOpt algorithm.} equaling {1, 2} is also employed as a baseline. The purpose of including these methods is to assess the effectiveness of the ``local-searching'' scheme utilized in TnALE (our proposed method) and determine its superiority in comparison to existing approaches. 
In more detail, for TR-SVD, TR-rSVD, TR-ALSAR, TR-BALS, and TR-BALS2~\cite{zhao2016tensor}, the available codes have been used.\footnote{https://qibinzhao.github.io/} In order to achieve a larger \textit{Eff.} value, we adjust the parameters \textit{tol} and \textit{MaxIter} to ensure the value of RSE is less than but close to $10^{-4}$. For TR-LM (Alg. 2 and Alg. 3)~\cite{mickelin2020algorithms}, we use the available codes \footnote{https://github.com/oscarmickelin/tensor-ring-decomposition} with default parameter settings. However, we adjust the value of \textit{prec} to obtain a larger \textit{Eff.} value. For TRAR~\cite{sedighin2021adaptive}, we replace the TR-ALS~\cite{wang2017efficient} in \textit{Algorithm 1} of ~\citet{mickelin2020algorithms} with the same decomposition method used in TTOpt. This modification is necessary because the initialization method of TR-ALS is not suitable for the case of higher ranks. Regarding TTOpt~\cite{sozykin2022ttopt}, we employ the same objective function as used in the TN-PS experiment, with the  trade-off parameter $\lambda=200$. For the \textit{lower ranks} group, the rank searching range is set to $[1,7]$, while for the \textit{higher ranks} group, the range is extended to $[1,10]$. During the initialization phase, we \textit{i.i.d.} draw samples from Gaussian distribution \textit{N} (0, 1) to generate the values of core tensors. For the proposed method TnALE, we set the rank-related radius $r_{1}=3, r_{2}=2$ for the higher ranks group and  $r_{1}=2, r_{2}=1$ for lower ranks group. The number of iterations in the initialization phase is set to 1, the number of iterations in the searching phase is set to 30, and the number of round-trips of ALE is set to 1 throughout the experiments. Other parameters of TnALE are set the same as TTOpt. For TNGA, we set the population in each generation to be 60. The searching ranges and the initialization scheme of core tensors are similar to TTOpt. The other parameters of TNGA are set the same as the TN-PS experiment. For TNLS, we set the sample numbers in each local sampling stage to be 60 and $c_{1}=0.9$, and the other parameters are set the same as in TTOpt.
The success condition for all approaches in the experiment is set as RSE $\le 10^{-4}$ and \textit{Eff.} $\ge 1$. If an approach fails to meet these criteria, it is considered a failure in rank selection.

\begin{figure}[ht]
    \centering
    \includegraphics[width=0.6\columnwidth]{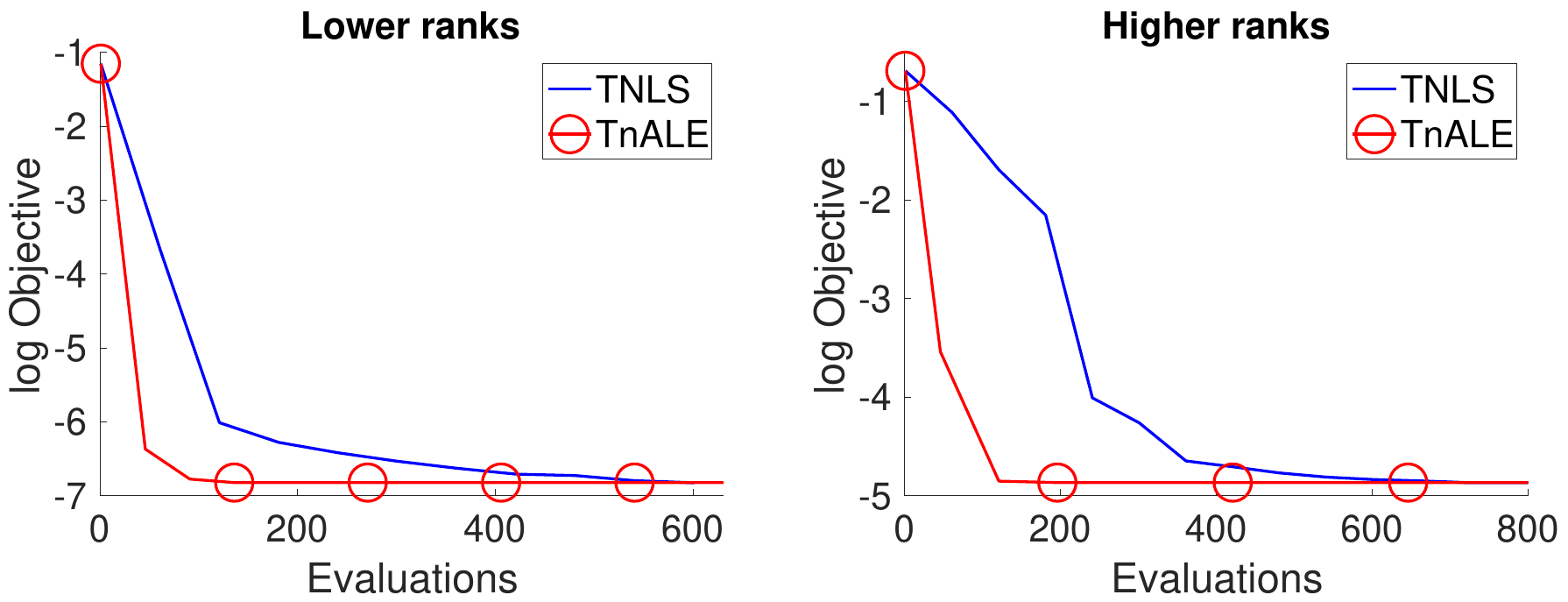}
    \vspace{-0.1cm}
    \caption{Objective (in the log form) with varying the number of evaluations.}
    \label{apd:fig:TNRS}
\end{figure}
\begin{figure}[ht]
    \centering
    \includegraphics[width=0.4\columnwidth]{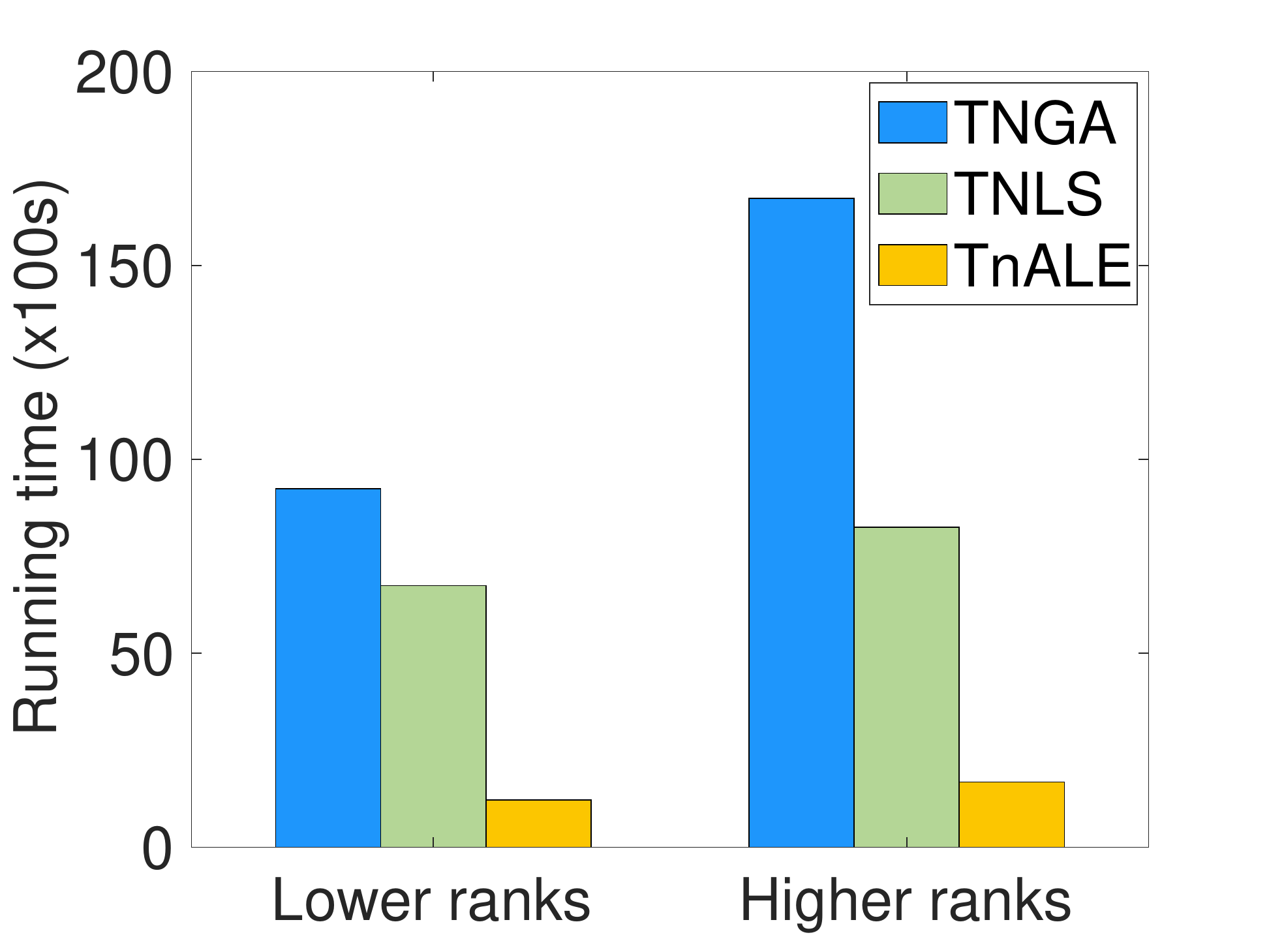}
    \vspace{-0.2cm}
    \caption{Running time in TN-RS experiment}
    \label{apd:bar:RunningtimeTNRS}
\end{figure}

\begin{table}[t]
\caption{Experimental results of TN-RS~(rank selection) in 8-th order TR topology under the "lower ranks" group. In the first column of the table, A, B, C, D, E~(\textit{Data}) and their corresponding vectors~(\textit{Rank\_grt}) represent the five generated synthetic tensors and the TN-ranks of these five tensors. The item \textit{{Rank\_est}} indicates the specific value of the TN-ranks learned by the corresponding method under the constraint RSE $\leq 10^{-4}$, and \textit{Time (s)} or \textit{[\#Eva.]} indicates the running time or the number of evaluations that the method required.}
\vskip 0.15in
\label{apd:tab:lowerrank8}
\centering
\begin{tiny}
    \begin{tabular}{cccccccccc}
    \toprule
    \textbf{Methods} & \multicolumn{3}{c}{\textbf{TR-SVD}} & \multicolumn{3}{c}{\textbf{TR-rSVD}} & \multicolumn{3}{c}{\textbf{TR-ALSAR}} \\
    \midrule
    \textit{Data[Rank\_grt]} & \textit{Eff.}~/~RSE & \textit{Rank\_est} & \textit{Time~(s)} & \textit{Eff.}~/~RSE & \textit{Rank\_est} & \textit{Time~(s)} & \textit{Eff.}~/~RSE & \textit{Rank\_est} & \textit{Time~(s)} \\
    A [3 4 2 3 1 3 4 2] & 0.45~/~8.55E-13 & [3 8 4 6 2 6 3 1] & 0.0029 & 0.51~/~0.0013 & [3 6 4 6 2 6 3 1] & 0.0038 & 0.08~/~2.43E-05 & [11 7 7 7 10 14 8 10] & 0.5959 \\
    B [3 4 4 2 2 1 1 4] & 0.23~/~4.45E-05 & [3 9 11 6 6 3 3 1] & 0.0081 & 0.37~/~0.0146 & [3 6 6 6 6 3 3 1] & 0.0043 & 0.79~/~4.02E-12 & [4 4 4 2 2 2 3 3] & 0.0324 \\
    C [2 4 2 3 3 1 4 4] & 0.29~/~3.55E-05 & [3 9 8 4 8 4 3 1] & 0.002812 & 0.46~/~0.0210 & [3 6 6 4 6 3 3 1] & 0.0039 & 0.94~/~3.55E-05 & [4 4 2 1 2 2 3 4] & 0.0333 \\
    D [1 4 1 3 4 2 1 1] & \textbf{1.13~/~4.78E-05} & [1 2 1 3 4 2 1] & 0.0084 & \textbf{1.13~/~4.78E-05} & [1 2 1 3 4 2 1 1] & 0.0055 & 0.64~/~2.29E-11 & [1 3 3 4 4 2 2 1] & 0.0233 \\
    E [4 1 4 2 3 2 1 1] & \textbf{1.17~/~2.04E-13} & [3 1 3 2 3 2 1 1] & 0.0043 & \textbf{1.17~/~3.13E-13} & [3 1 3 2 3 2 1 1] & 0.0059 & 0.78~/~1.45E-11 & [3 3 3 2 3 2 2 1] & 0.0205 \\
    \midrule
    \textbf{Methods} & \multicolumn{3}{c}{\textbf{TR-BALS}} & \multicolumn{3}{c}{\textbf{TR-BALS2}} & \multicolumn{3}{c}{\textbf{TRAR}} \\
    \midrule
    \textit{Data} & \textit{Eff.}~/~RSE & \textit{Rank\_est} & \textit{Time~(s)} & \textit{Eff.}~/~RSE & \textit{Rank\_est} & \textit{Time~(s)} & \textit{Eff.}~/~RSE & \textit{Rank\_est} & \textit{[\#Eva.]} \\
    A     & \textbf{1.00~/~5.00E-13} & [3 4 2 3 1 3 4 2] & 0.0146 & 0.45~/~6.90E-13 & [3 8 4 6 2 6 3 1] & 0.0188 & 0.48~/~3.79E-11 & [4 4 3 4 5 4 4 3] & 69 \\
    B     & \textbf{1.07~/~6.07E-13} & [3 4 4 2 2 1 1 3] & 0.0127 & 0.20~/~1.03E-05 & [3 9 12 6 6 3 3 5] & 0.0174 & 0.61~/~3.17E-08 & [3 4 4 4 4 2 4 3] & 37 \\
    C     & \textbf{1.38~/~3.55E-05} & [2 4 2 1 2 1 3 4] & 0.0189 & 0.26~/~3.57E-05 & [3 9 8 4 8 4 3 6] & 0.0151 & 0.65~/~3.55E-05 & [3 4 3 5 3 3 3 4] & 68 \\
    D     & \textbf{1.13~/~4.78E-05} & [1 2 1 3 4 2 1 1] & 0.021 & \textbf{1.13~/~4.78E-05} & [1 2 1 3 4 2 1 1] & 0.0487 & 0.41~/~5.55E-05 & [4 5 3 3 3 3 3 2] & 22 \\
    E     & \textbf{1.17~/~6.34E-13} & [3 1 3 2 3 2 1 1] & 0.0303 & \textbf{1.17~/~9.10E-13} & [3 1 3 2 3 2 1 1] & 0.0212 & 0.59~/~2.96E-11 & [5 2 3 2 3 2 2 3] & 36 \\
    \midrule
    \textbf{Methods} & \multicolumn{3}{c}{\textbf{TR-LM (Alg. 3)}} & \multicolumn{3}{c}{\textbf{TR-LM (Alg. 2)}} & \multicolumn{3}{c}{\textbf{TTOpt}~(\textit{R} = 1)} \\
    \midrule
    \textit{Data} & \textit{Eff.}~/~RSE & \textit{Rank\_est} & \textit{Time~(s)} & \textit{Eff.}~/~RSE & \textit{Rank\_est} & \textit{Time~(s)} & \textit{Eff.}~/~RSE & \textit{Rank\_est} & \textit{[\#Eva.]} \\
    A     & 0.40~/~4.22E-13 & [6 3 1 3 2 6 8 4] & 0.0222 & \textbf{1.00~/~2.86E-14} & [3 4 2 3 1 3 4 2] & 0.2736 & \textbf{1.00~/~9.41E-07} & [3 4 2 3 1 3 4 2] & 98 \\
    B     & 0.46~/~2.05E-06 & [6 7 3 1 3 2 2 6] & 0.0204 & \textbf{1.07~/~1.08E-14} & [3 4 4 2 2 1 1 3] & 0.2402 & \textbf{1.07~/~2.30E-06} & [3 4 4 2 2 1 1 3] & 140 \\
    C     & \textbf{1.38~/~3.55E-05} & [2 4 2 1 2 1 3 4] & 0.0227 & \textbf{1.38~/~3.55E-05} & [2 4 2 1 2 1 3 4] & 0.2503 & \textbf{1.11~/~3.55E-05} & [2 4 2 1 2 2 4 4] & 56 \\
    D     & \textbf{1.13~/~4.78E-05} & [1 2 1 3 4 3 1 1] & 0.026 & \textbf{1.13~/~4.78E-05} & [1 2 1 3 4 2 1 1] & 0.2407 & \textbf{1.06~/~4.82E-05} & [1 2 1 3 4 2 1 2] & 91 \\
    E     & \textbf{1.17~/~5.62E-14} & [3 1 3 2 3 2 1 1] & 0.022 & \textbf{1.17~/~5.62E-14} & [3 1 3 2 3 2 1 1] & 0.2454 & \textbf{1.17~/~6.61E-11} & [3 1 3 2 3 2 1 1] & 133 \\
    \midrule
    \textbf{Methods} & \multicolumn{3}{c}{\textbf{TTOpt}~(\textit{R} = 2)} & \multicolumn{3}{c}{\textbf{TTOpt}~(\textit{R} = 3)} & \multicolumn{3}{c}{\textbf{TNGA}} \\
    \midrule
    \textit{Data} & \textit{Eff.}~/~RSE & \textit{Rank\_est} & \textit{[\#Eva.]} & \textit{Eff.}~/~RSE & \textit{Rank\_est} & \textit{[\#Eva.]} & \textit{Eff.}~/~RSE & \textit{Rank\_est} & \textit{[\#Eva.]} \\
    A   & \textbf{1.00~/~5.00E-06} &    [3 4 2 3 1 3 4 2]   & 518   
    & \textbf{1.00~/~8.93E-05} &    [3 4 2 3 1 3 4 2]   & 1533   
     & \textbf{1.00~/~9.98E-05} &  [3 4 2 3 1 3 4 2]     & 480\\
    B  & \textbf{1.02~/~4.86E-07} &   [3 4 4 2 2 2 1 3]    & 336       
    & \textbf{1.02~/~3.72E-06} &   [3 4 4 2 2 2 1 3]    & 735   
    & \textbf{1.07~/~9.95E-05} &    [3 4 4 2 2 1 1 3]   & 660\\
    C   & \textbf{1.02~/~3.56E-05} &   [2 5 2 2 3 1 3 5]    & 154    
    & \textbf{1.00~/~9.45E-05} &   [2 5 2 2 3 2 3 4]    & 273   
     & \textbf{1.11~/~9.95E-05} &   [2 4 2 1 2 2 4 4]    & 600\\
    D  & \textbf{1.06~/~1.10E-08} &    [1 3 1 3 4 2 1 1]   & 196     
    & \textbf{1.00~/~4.87E-05} &  [1 2 1 3 4 2 1 3]     & 483   
     & \textbf{1.06~/~9.98E-05} &   [1 2 1 3 4 2 1 2]    & 600\\
    E    & \textbf{1.03~/~7.94E-06} &    [3 1 3 2 3 3 1 1]   & 364     
    & \textbf{1.17~/~3.14E-11} &  [3 1 3 2 3 2 1 1]     & 1071   
    & \textbf{1.17~/~9.92E-05} &   [3 1 3 2 3 2 1 1]    & 420\\
    \midrule
    \textbf{Methods} &     \multicolumn{3}{c}{\textbf{TNLS}} &     \multicolumn{3}{c}{\textbf{TnALE (ours)}}\\
    \cmidrule{1-7}  \textit{Data} & \textit{Eff.}~/~RSE & \textit{Rank\_est} & \textit{[\#Eva.]} & \textit{Eff.}~/~RSE & \textit{Rank\_est} & \textit{[\#Eva.]}\\
    A     & \textbf{1.00~/~9.99E-05} &    [3 4 2 3 1 3 4 2]   & 600 
    & \textbf{1.00~/~9.98E-05} &    [3 4 2 3 1 3 4 2]   & 66 \\
    B     &\textbf{1.07~/~9.97E-05} &   [3 4 4 2 2 1 1 3]    & 420
    & \textbf{1.07~/~9.98E-05} &   [3 4 4 2 2 1 1 3]    & 99 \\
    C     & \textbf{1.11~/~9.99E-05} &   [2 4 2 1 2 2 4 4]    & 420 
    & \textbf{1.11~/~9.95E-05} &   [2 4 2 1 2 2 4 4]    & 99 \\
    D     & \textbf{1.06~/~9.97E-05} &  [1 2 1 3 4 2 1 2]     & 480
    & \textbf{1.06~/~9.98E-05} &    [1 2 1 3 4 2 1 2]   & 69 \\
    E    & \textbf{1.17~/~9.92E-05} &  [3 1 3 2 3 2 1 1]     & 540  
    & \textbf{1.17~/~9.93E-05} &    [3 1 3 2 3 2 1 1]   & 63 \\
    \bottomrule
    \end{tabular}%
  \end{tiny}
\vskip -0.1in
\end{table}

\begin{table}[t]
\caption{Experimental results of TN-RS~(rank selection) in 8-th order TR topology under the "higher ranks" group. In the first column of the table, A, B, C, D, E~(\textit{Data}) and their corresponding vectors~(\textit{Rank\_grt}) represent the five generated synthetic tensors and the TN-ranks of these five tensors. The item \textit{{Rank\_est}} indicates the specific value of the TN-ranks learned by the corresponding method under the constraint RSE $\leq 10^{-4}$, and \textit{Time (s)} or \textit{[\#Eva.]} indicates the running time or the number of evaluations that the method required.}
\vskip 0.15in
\label{apd:tab:higherrank8}
\centering
  \begin{tiny}
    \begin{tabular}{cccccccccc}
    \toprule
    \textbf{Methods} & \multicolumn{3}{c}{\textbf{TR-SVD}} & \multicolumn{3}{c}{\textbf{TR-rSVD}} & \multicolumn{3}{c}{\textbf{TR-ALSAR}} \\
    \midrule
    \textit{Data[Rank\_grt]} & \textit{Eff.}~/~RSE & \textit{Rank\_est} & \textit{Time~(s)} & \textit{Eff.}~/~RSE & \textit{Rank\_est} & \textit{Time~(s)} & \textit{Eff.}~/~RSE & \textit{Rank\_est} & \textit{Time~(s)} \\
    A [8 8 8 5 7 6 8 8] & 0.16~/~9.40E-05 & [3 9 27 38 27 9 3 1] & 0.0174 & 0.16~/~9.40E-05 & [3 9 27 38 27 9 3 1] & 0.0238 & 1.11~/~0.0172 & [7 7 7 6 6 6 8 8] & 11.2686 \\
    B [6 5 7 7 6 5 6 5] & 0.12~/~4.30E-05 & [3 9 27 35 26 9 3 1] & 0.0136 & 0.11~/~1.13E-28 & [3 9 27 35 27 9 3 1] & 0.0361 & 0.82~/~0.0133 & [7 5 6 8 6 6 8 6] & 10.536 \\
    C [8 7 7 8 7 5 8 7] & 0.12~/~8.32E-05 & [3 9 27 51 27 9 3 1] & 0.0378 & 0.12~/~8.32E-05 & [3 9 27 51 27  9 3 1] & 0.0308 & 0.71~/~0.0094 & [9 8 6 17 6 7 9 8] & 16.2148 \\
    D [6 6 6 8 6 7 6 5] & 0.13~/~9.64E-05 & [3 9 26 38 26 9 3 1] & 0.0092 & 0.12~/~5.32E-05 & [3 9 27 38 27 9 3 1] & 0.0308 & 0.63~/~9.62E-05 & [9 8 8 8 7 9 7 7] & 0.6589 \\
    E [6 6 6 6 5 6 6 6] & 0.11~/~4.51E-05 & [3 9 27 36 26 9 3 1] & 0.0061 & 0.11~/~5.63E-05 & [3 9 27 35 27 9 3 1] & 0.0242 & 0.75~/~0.0298 & [7 7 5 6 4 9 8 8] & 11.5233 \\
    \midrule
    \textbf{Methods} & \multicolumn{3}{c}{\textbf{TR-BALS}} & \multicolumn{3}{c}{\textbf{TR-BALS2}} & \multicolumn{3}{c}{\textbf{TRAR}} \\
    \midrule
    \textit{Data} & \textit{Eff.}~/~RSE & \textit{Rank\_est} & \textit{Time~(s)} & \textit{Eff.}~/~RSE & \textit{Rank\_est} & \textit{Time~(s)} & \textit{Eff.}~/~RSE & \textit{Rank\_est} & \textit{[\#Eva.]} \\
    A     & 0.03~/~6.04E-29 & [28 20 19 26 44 89 51 39] & 0.8749 & 0.16~/~8.39E-05 & [3 9 27 39 27 9 3 1] & 0.142 & 0.67~/~5.59E-14 & [10   8   8   8   8  10   8  11] & 76 \\
    B     & 0.57~/~9.25E-05 & [6 6 10 9 9 7 9 6] & 0.0941 & 0.12~/~9.89E-05 & [3 9 27 35 26 9 3 1] & 0.1608 & 0.71~/~1.82E-13 & [8  5  7  7  7  7  6  9] & 74 \\
    C     & 0.12~/~9.71E-05 & [15 20 18 15 17 19 23 25] & 0.2763 & 0.12~/~6.35E-05 & [3 9 27 54 27 9 3 1] & 0.1643 & 0.58~/~8.89E-14 & [12   7  10   8   9  10   9  10] & 143 \\
    D     & 0.19~/~8.18E-05 & [9 15 17 22 10 16 15 10] & 0.1773 & 0.12~/~9.75E-06 & [3 9 27 40 26 9 3 1] & 0.1704 & 0.57~/~3.75E-14 & [11   6   7   8   8   9   7  10] & 76 \\
    E     & 0.03~/~3.97E-05 & [38 55 59 20 18 15 19 27] & 0.446 & 0.11~/~6.40E-30 & [3 9 27 36 27 9 3 1] & 0.1749 & 0.62~/~2.05E-14 & [10   6   7   7   7   8   6   9] & 75 \\
    \midrule
    \textbf{Methods} & \multicolumn{3}{c}{\textbf{TR-LM (Alg. 3)}} & \multicolumn{3}{c}{\textbf{TR-LM (Alg. 2)}} & \multicolumn{3}{c}{\textbf{TTOpt}~(\textit{R} = 1)} \\
    \midrule
    \textit{Data} & \textit{Eff.}~/~RSE & \textit{Rank\_est} & \textit{Time~(s)} & \textit{Eff.}~/~RSE & \textit{Rank\_est} & \textit{Time~(s)} & \textit{Eff.}~/~RSE & \textit{Rank\_est} & \textit{[\#Eva.]} \\
    A     & 0.16~/~9.40E-05 & [3 9 27 38 27 9 3 1] & 0.0257 & 0.16~/~2.87E-05 & [3 9 27 39 27 9 3 1] & 0.3318 & \textbf{1.00~/~3.65E-07} & [8  8  8  5  7  6  8  8] & 220 \\
    B     & 0.12~/~4.30E-05 & [3 9 27 35 26 9 3 1] & 0.001 & 0.15~/~8.36E-05 & [3 1 3 9 26 25 25 9] & 0.3336 & \textbf{1.00~/~1.52E-07} & [6  5  7  7  6  5  6  5] & 220 \\
    C     & 0.12~/~8.32E-05 & [3 9 27 51 27 9 3 1] & 0.0303 & 0.17~/~7.07E-05 & [3 1 3 9 27 34 27 9] & 0.3285 & \textbf{1.00~/~1.01E-06} & [8  7  7  8  7  5  8  7] & 150 \\
    D     & 0.13~/~9.64E-05 & [3 9 26 38 26 9 3 1] & 0.023 & 0.13~/~9.31E-05 & [35 27 9 3 1 3 9 25] & 0.3173 & \textbf{1.00~/~1.83E-06} & [6  6  6  8  6  7  6  5] & 150 \\
    E     & 0.11~/~4.51E-05 & [3 9 27 36 26 9 3 1] & 0.0299 & 0.13~/~9.38E-05 & [20 26 9 3 1 3 9 26] & 0.3845 & \textbf{1.00~/~5.01E-07} & [6  6  6  6  5  6  6  6] & 150 \\
    \midrule
    \textbf{Methods} & \multicolumn{3}{c}{\textbf{TTOpt}~(\textit{R} = 2)} & \multicolumn{3}{c}{\textbf{TTOpt}~(\textit{R} = 3)} & \multicolumn{3}{c}{\textbf{TNGA}} \\
    \midrule
    \textit{Data} & \textit{Eff.}~/~RSE & \textit{Rank\_est} & \textit{[\#Eva.]} & \textit{Eff.}~/~RSE & \textit{Rank\_est} & \textit{[\#Eva.]} & \textit{Eff.}~/~RSE & \textit{Rank\_est} & \textit{[\#Eva.]} \\
    A 
        & \textbf{1.00~/~6.49E-07} & [8  8  8  5  7  6  8  8] & 540    
        & \textbf{1.00~/~5.88E-07} & [8  8  8  5  7  6  8  8] & 1710   
        & \textbf{1.00~/~9.99E-05} & [8  8  8  5  7  6  8  8] & 1020\\
    B  
        & \textbf{1.00~/~1.61E-07} & [6  5  7  7  6  5  6  5] & 1060      
        & \textbf{1.00~/~4.00E-07} & [6  5  7  7  6  5  6  5] & 2670   
        & \textbf{1.00~/~9.94E-05} & [6  5  7  7  6  5  6  5] & 660 \\
    C  
        & \textbf{1.00~/~1.32E-06} & [8  7  7  8  7  5  8  7] & 780     
        & \textbf{1.00~/~9.23E-07} & [8  7  7  8  7  5  8  7] & 1740   
        & \textbf{1.00~/~9.95E-05} & [8  7  7  8  7  5  8  7] & 1380\\
    D  
        & \textbf{1.00~/~3.40E-07} & [6  6  6  8  6  7  6  5] & 540      
        & \textbf{1.00~/~1.67E-06} & [6  6  6  8  6  7  6  5] & 1710   
        & \textbf{1.00~/~9.93E-05} & [6  6  6  8  6  7  6  5] & 1020\\
    E    
        & \textbf{1.00~/~8.02E-13} & [6  6  6  6  5  6  6  6] & 840     
        & \textbf{1.00~/~1.99E-07} & [6  6  6  6  5  6  6  6] & 1740   
        & \textbf{1.00~/~9.94E-05} & [6  6  6  6  5  6  6  6] & 420\\
    \midrule
    \textbf{Methods} &     \multicolumn{3}{c}{\textbf{TNLS}} &     \multicolumn{3}{c}{\textbf{TnALE (ours)}}\\
    \cmidrule{1-7}  \textit{Data} & \textit{Eff.}~/~RSE & \textit{Rank\_est} & \textit{[\#Eva.]} & \textit{Eff.}~/~RSE & \textit{Rank\_est} & \textit{[\#Eva.]}\\
    A  
        & \textbf{1.00~/~9.97E-05} & [8  8  8  5  7  6  8  8] & 540 
        & \textbf{1.00~/~9.96E-05} & [8  8  8  5  7  6  8  8] & 115 \\
    B   
        & \textbf{1.00~/~9.92E-05} & [6  5  7  7  6  5  6  5] & 720 
        & \textbf{1.00~/~9.91E-05} & [6  5  7  7  6  5  6  5] & 85 \\
    C   
        & \textbf{1.00~/~9.91E-05} & [8  7  7  8  7  5  8  7] & 480
        & \textbf{1.00~/~1.00E-04} & [8  7  7  8  7  5  8  7] & 150 \\
    D  
        & \textbf{1.00~/~9.93E-05} & [6  6  6  8  6  7  6  5] & 600 
        & \textbf{1.00~/~9.92E-05} & [6  6  6  8  6  7  6  5] & 160 \\
    E   
        & \textbf{1.00~/~9.92E-05} & [6  6  6  6  5  6  6  6] & 600
        & \textbf{1.00~/~9.93E-05} & [6  6  6  6  5  6  6  6] & 85 \\   
    \bottomrule
    \end{tabular}%
\end{tiny}
\vskip -0.1in
\end{table}%

\textbf{Results.} 
Based on the results presented in Table~\ref{apd:tab:lowerrank8} and Table~\ref{apd:tab:higherrank8}, it can be observed that in the lower rank regime, TR-BALS, TR-LM (Alg. 2), TTOpt, TNGA, TNLS, and TnALE~(ours) are able to successfully select the optimal TR-ranks as indicated by \textit{Eff.}$\geq{}1$ and RSE$\leq 10^{-4}$. However, in the higher rank regime, only TTOpt, TNGA, TNLS, and TnALE~(ours) are able to find the optimal ranks. In terms of the number of evaluations, TnALE~(ours) outperforms TNGA, TNLS, and TTOpt, requiring the fewest evaluations while still achieving successful rank selection. This highlights the superiority of TnALE in solving the TN-RS problem efficiently. Furthermore, the running time comparison in Figure ~\ref{apd:bar:RunningtimeTNRS} demonstrates that TnALE saves a significant amount of time compared to TNGA and TNLS, primarily due to its lower number of evaluations. This further emphasizes the advantage of TnALE in scenarios where computational resources are limited. In Figure ~\ref{apd:fig:TNRS}, the averaged log objective curves of TNLS and TnALE with varying evaluation numbers are illustrated. It can be observed that TnALE exhibits a faster descending trend and achieves lower objective values given the same number of evaluations compared to TNLS. This demonstrates the practical advantage of TnALE, particularly in scenarios with restricted computational resources.

\subsection{Details for the experiment of knowledge transfer.}
\begin{figure}[ht]
    \centering
    \includegraphics[width=1\columnwidth]{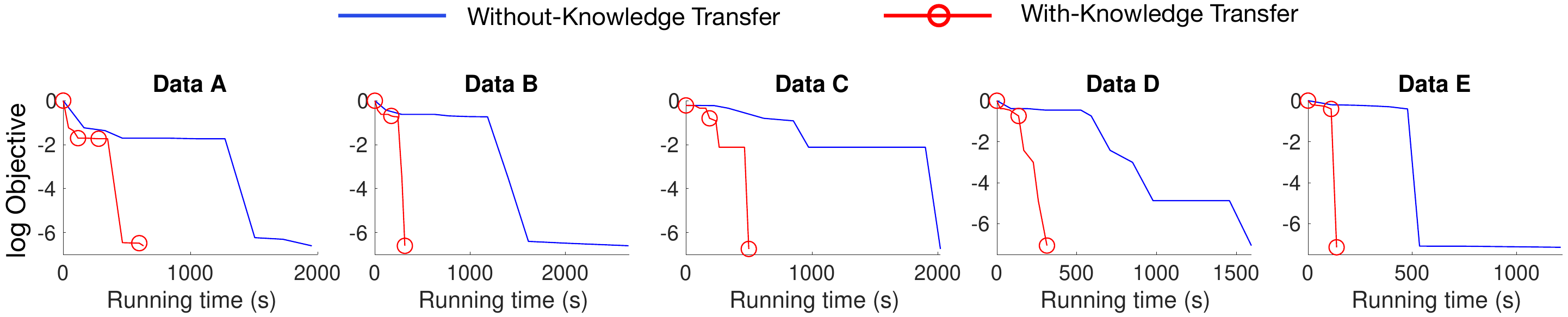}
    \vspace{-0.4cm}
    \caption{Objective (in the log form) curves with running time}
    \label{apd:fig:KT}
\end{figure}
\textbf{Goal.}
In this experiment, the goal is to investigate the acceleration effect of TnALE when employing the knowledge transfer trick.

\textbf{Data generation.}
We \emph{re-use} the data from the lower ranks group of the TN-RS experiment.

\textbf{Settings.}
In this experiment, we employ two variations of TnALE: one incorporates a knowledge transfer trick, while the other does not. Both methods share the same parameter settings, which are listed as follows: the rank searching range is set to $[1, 7]$, the trade-off parameter $\lambda$ is set to 200, the rank-related radius $r_{2}=2$. Additionally, we set the number of iterations in the initialization phase to 0 and the number of iterations in the searching phase to 30. For the number of round-trips of ALE, we set it to 1. The Adam optimizer is utilized with a learning rate of 0.001, and the core tensors are initialized using Gaussian distribution  $N(0, 1)$. Moreover, both methods are initialized with the same TN-ranks.

\textbf{Results.}
Figure ~\ref{apd:fig:KT} displays the objective curves as a function of running time. From the figures, it is evident that both methods start with identical log objectives but exhibit significant differences in their descent patterns. In comparison to TnALE without the knowledge transfer trick, TnALE with the knowledge transfer trick showcases a rapid decline in objectives, achieving approximately twice or even nearly five times faster progress than its counterpart.

\subsection{Details for the experiment of TGP (\textit{w.r.t.} Table~\ref{tab:GP}).}

\textbf{Goal.}
In this experiment, our goal is to utilize the proposed method TnALE to compress the learnable parameters of the TGP ~\cite{izmailov2018scalable}.

\textbf{Data generation.}
In this task, we select three univariate regression datasets from the UCI and LIBSVM archives. The datasets chosen are as follows: The Combined Cycle Power Plant (CCPP)\footnote{https://archive.ics.uci.edu/ml/datasets/Combined+Cycle+Power+Plant} dataset comprises 9569 data points collected from a power plant. It consists of 4 features and a single response. The MG\footnote{https://www.csie.ntu.edu.tw/~cjlin/libsvmtools/datasets/regression.html\#mg} dataset contains 1385 data points with 6 features and a single response.  The Protein\footnote{https://archive.ics.uci.edu/ml/datasets/Physicochemical+Properties+of+Protein+Tertiary+Structure} dataset consists of 45730 instances with 9 attributes and a single response. For each of the datasets, we begin by randomly selecting 80\% of the data for training purposes, while the remaining 20\% is reserved for testing. Subsequently, we standardize the training and testing sets respectively by removing the mean and scaling them to have unit variance. In the case of the CCPP dataset, we opt to use 12 inducing points on each feature, resulting in an order-4 tensor with dimensions of $12 \times 12 \times 12 \times 12$. For the MG dataset, we choose 8 inducing points, which leads to an order-6 tensor with dimensions of $8 \times 8 \times 8 \times 8 \times 8 \times 8$. Lastly, for the Protein dataset, we choose 4 inducing points, generating an order-9 tensor with dimensions of $4 \times 4 \times 4 \times 4 \times 4 \times 4 \times 4 \times 4 \times 4$. Across all datasets, we set the TT-ranks for the TGP~\cite{izmailov2018scalable} algorithm to 10. 

\textbf{Settings.}
In the comparison of methods, we employ the same objective function as used in the TN-PS experiment. Additionally, we set specific values for certain parameters, $\lambda=1\times10^{5},1\times10^{7},1\times10^{3}$ for CCPP, MG and Protein, respectively. Moreover, the following settings are common for all the methods being compared: the rank searching range, the learning rate of Adam, and the variance of the Gaussian distribution for core tensors initialization are set from 1 to 14, 0.001, and 0.01, respectively. For the TNGA method, we set the maximum number of generations to 30. The population in each generation is set to be 150, 190, and 300 for the TT variational mean of CCPP, MG, and Protein regression tasks. The elimination rate is set at $30\%$ and the reproduction number is set to 1. Additionally, we assign $\alpha = 20$ and $\beta = 1$. The chance for each gene to mutate after the recombination is $30\%$. For TNLS, the maximum iteration is limited to 20, and the tuning parameters $c_{1}$ = 0.9, $c_{2}$ = 0.9. For the TT variational mean of CCPP, MG, and Protein regression tasks, we determined the number of samples in the local sampling stage to be 150, 300, and 300 respectively.
For the proposed method TnALE, we consistently use the rank-related radius $r_{1}=3$ and $r_{2}=2$. In addition, we specifically designate the number of iterations in the initialization phase as 2 and the number of iterations in the searching phase as 30. Furthermore, we configure the number of round-trips in ALE to be 1.

In order to achieve more compact representations, we apply the TN-PS algorithms, which consist of TNGA, TNLS, and the proposed TnALE to TGP. The process involves training an initial TGP model with predefined TT-ranks and obtaining the TT representation of the variational mean.  Subsequently, the TN-PS algorithms are employed to search for alternative structures that have a reduced number of parameters for the TT variational mean. Upon completion of the TN-PS algorithms, we reintegrate the reparameterized variational mean back into the original TGP model for inference. The performance is evaluated by measuring the mean squared error (MSE) of the regression tasks conducted on the test datasets.


\subsection{Details for the experiment of natural images compression (\textit{w.r.t.} Table~\ref{tab:TNTS}).}

\textbf{Goal.}
In this experiment, we will investigate the effectiveness of the proposed TnALE method in tackling the TN-PS and TN-TS tasks associated with compressing natural images. Specifically, in TN-TS, our aim is to search for good TN-ranks and topologies for compressing images.

\textbf{Data generation.}
For this experiment, we select 4 natural images from the BSD500 dataset \cite{arbelaez2010contour}\footnote{https://www2.eecs.berkeley.edu/Research/Projects/CS/vision/bsds/BSDS300/html/dataset/images.html} at random, as shown in Figure \ref{fig:livecompression}. The selected images are first converted to grayscaled images of size $256\times256$ using the "rgb2gray" and ``resize'' functions in Matlab, then scaled to the range of $\lbrack0,1\rbrack$. Finally, we apply the Matlab function "reshape" directly to the preprocessed images to represent them as order-8 tensors of size $4 \times 4 \times 4 \times 4 \times 4 \times 4 \times 4 \times 4$.

\textbf{Settings.}
In the TN-PS task of the experiment, we use the same objective function as in the TN-PS experiment and set the tuning parameter $\lambda=5$. The rank searching range, the learning rate of Adam, and the variance of the Gaussian distribution for core tensors initialization are set from 1 to 14, 0.01, and 0.1, respectively. For TNLS, we set the maximum number of iterations to 20, and tuning parameters $c_{1}$ = 0.95, $c_{2}$ = 0.9, and the number of samples in the local sampling stage to 150. In TNGA, the maximum number of generations is set to 30, with a population of 300 per generation. The elimination is set at $10\%$ and the reproduction number is set to 1. We also set $\alpha$ = 25, $\beta$ = 1, and establish a $30\%$ chance for each gene to mutate following the recombination process. Regarding the proposed method TnALE, we set the rank-related radius as $r_{1}=3$ and $r_{2}=2$. We also set the number of iterations in the initialization phase to 1 and the number of iterations in the searching phase to 30. Finally, we set the number of round-trips of ALE to 1.

In the TN-TS task of the experiment, we use the same objective function, the learning rate of Adam, and the variance of the Gaussian distribution for core tensors initialization as in the TN-PS part, but set the rank searching range from 1 to 4. For TNLS, we set the maximum number of iterations to 20, tuning parameters $c_{1}$ = 0.99, and the number of samples in the local sampling stage to 100. For the parameter settings of TNGA, we only change the population number to 100 compared to the TN-PS part. For the proposed method TnALE, we set the rank-related radius $r_{2}$ to 1 and the number of iterations in the initialization phase to 0, while the number of iterations in the searching phase to 30. The number of round-trips of ALE is also set to 1. For Greedy, we set the RSE threshold to the same value as the result RSE of the proposed method TnALE.



\end{document}